\documentclass{article}

\usepackage{times}
\usepackage{tda}
\usepackage[margin=1.25in]{geometry}

\makeatletter
\let\@fnsymbol\@arabic
\makeatother

\begin{document}

\title{The Manifold Density Function: An Intrinsic Method for the Validation of Manifold Learning}
\author{ 
  Benjamin Holmgren 
    \thanks{Department of Computer Science, Duke University}, \and 
  Eli Quist 
    \thanks{School of Computing, Montana State University}~
    \thanks{Department of Mathematical Sciences, Montana State University} , \and 
  Jordan Schupbach \samethanks[3], \and  
  Brittany Terese Fasy 
    \samethanks[2] 
    ~\samethanks[3], \and 
  Bastian Rieck 
    \thanks{Helmholtz Munich and Technical University Munich}
}
\date{}

\maketitle

\begin{abstract}%
  % \bastian{%
%   Abstract is great and concise. What does \emph{intrinsic} mean?%
% }
% \brittany{good point. we don't define that anywhere, do we?}

We introduce the \ourkf, which is an intrinsic method to validate manifold learning techniques.
Our approach adapts and extends Ripley's \ripkf,
and categorizes in an unsupervised setting the extent to which an output of a manifold learning
algorithm captures the structure of a latent manifold.
Our \ourkf generalizes to
broad classes of Riemannian manifolds.
In particular, we extend the \ourkf to general two-manifolds using
the Gauss-Bonnet theorem, and demonstrate that the \ourkf 
for hypersurfaces is well approximated using the first Laplacian eigenvalue. 
We prove desirable convergence and robustness properties.
\end{abstract}

\textbf{Keywords} ~~ manifold learning, algorithm validation, Ripley's $K$-function, hypersurfaces%, manifolds, %

\section{Introduction}\label{sec:intro}
Manifold learning is extremely well-studied in the machine learning,
computational geometry, and computational topology literatures \cite{hensel2021, huo2008, Zheng2009}.
The uses of manifold learning, while generally categorized as a means for nonlinear
dimensionality reduction \cite{belkin2003, gtm1998, hessian, roweis2000lle, tenenbaum2000isomap},
span application areas including shape recognition \cite{jin2009computing},
image recognition \cite[Ch.~4]{zheng2009statistical, carlsson2008local}, and motion planning \cite{dirafzoon2016geometric}.
Given the large number of techniques and their practical significance,
natural questions concerning their \emph{validation} are raised.
That is, when has an algorithm effectively learned manifold structure
within data?
Moreover, how does one rigorously compare the performance of different
manifold learning algorithms?
When validating dimensionality reduction algorithms, it is especially
pertinent to work in an unsupervised setting. In this paper, we thus
assume little or no knowledge of the ground truth in data (such as
knowledge of true geodesic distances).
Validation in this unsupervised setting is called \emph{intrinsic} validation,
as observations are reliant only on core properties of the presented data.

Despite its prevalence and pivotal role in data analysis,
validation techniques specific to manifold learning have been largely
unexplored.
A manifold analog to Precision-Recall is given in the experimental analysis of \cite[Section 3.5]{geodesicPrecisionRecall}.
However, this method requires knowledge of geodesics on the latent manifold and can therefore be considered extrinsic
(even though the manifold learning method proposed in \cite{geodesicPrecisionRecall} is unsupervised).
Our work is also related to extensions of Ripley's \ripkf for spaces more complicated than $\R^d$,
such as the extension
over $\mathbb{S}^d$ given in \cite{moller2016functional}. More recently, Ward et al. develop nonparametric
estimation of intensity functions point processes over Riemannian manifolds
\cite{ward2023nonparametric}. However, such methods again rely on knowledge of a ground truth
geodesic distance, which is infeasible for most contexts in unsupervised manifold learning.
An intrinsic approach is given in \cite{vonrohrscheidt2023topological} which uses persistent local homology
to detect singularities in point cloud data.
This paper takes a different perspective, giving a global assessment of the manifold properties of data
rooted in differential geometry. The advantage of the proposed approach is that it allows for
intrinsic validation while making theoretical robustness and complexity
guarantees on large classes of Riemannian manifolds.

\paragraph*{Contributions}
Intuitively, to say that discrete data exhibits ``manifold-like'' properties
could naturally be taken to mean that data locally resembles a uniform sample
in $\R^d$. This manuscript makes such a notion rigorous.
We examine local neighborhoods within a point cloud, scoring
how closely each neighborhood resembles a uniform sample in $\R^d$
without any knowledge of the ground truth manifold.
Our method is based on the natural idea that geodesic balls on
a manifold $\X$ should grow proportionately to balls in $\mathbb{R}^n$,
up to the curvature of $\X$.
In this way, our work can be thought of as
a manifold analog to the silhouette coefficient for clustering \cite{rousseeuw1987}.
Specifically, taking inspiration from Ripley's \ripkf \cite{kfcn}, we define a density estimator for
manifolds using principles from differential geometry, most notably the
Gauss-Bonnet theorem and hypersurface inequalities relating scalar
curvature to first Laplacian eigenvalue.
In particular, this paper presents the following:
\begin{enumerate}
    \item We introduce the \ourkf, $\fcntheo{\X}$, which maintains desirable convergence properties
    and is exactly computable when the scalar curvature is known.
    \item A robust, efficiently computable approximation for $\fcntheo{\X}$
        on two-manifolds with provable accuracy using the Euler characteristic
        when the latent manifold structure is unknown.
    \item A robust, efficiently computable approximation for $\fcntheo{\X}$ on
    hypersurfaces with provable accuracy using the first Laplacian eigenvalue
    when the latent manifold structure is unkown.
\end{enumerate}
The \emph{impact} of our work is two-fold, providing both an intrinsic manifold
validation method and a way to
assess the uniformity of data.

\section{Background}\label{sec:prelim}
We begin by providing the requisite definitions from geometry, topology, and statistics.

\subsection{Fundamental Definitions From Topology and Geometry}\label{sec:fundamental}

We begin with the definitions from topology and geometry.
For an additional discussion,
we direct readers to \appendref{def-diff}.
Recall that a metric space $(\X,d)$ is written as a set $\X$ paired with a distance metric
$$d:\X \times \X \to \R_{\geq 0}.$$
We write $\ball{d}{x}{r}$ to denote the open
ball of radius~$r$ centered at $x \in (\X,d)$.
For ease of notation, we often write $\X$, with the distance assumed.
We say that $\X$ is an \emph{$n$-dimensional
manifold} if,
at every point in $\X$, there exists a neighborhood homeomorphic
to the unit ball in $\R^n$.

In particular, we are interested in Riemannian manifolds in this work. Let $\X$ be a smooth manifold, and let
$d:\X \times \X \to \R_{\geq 0}$ be a distance. We say that $(\X, d)$ is a \emph{Riemannian manifold}
if $d$ is a Riemannian metric.
In this case, $\X$ has a well-defined Lebesgue measure~$\mu$ \cite[\textsection
VII.1]{amann2005analysis}, making $(\X,d,\mu)$ a metric measure space.
For example,
shortest-path \emph{geodesic} distances on $\X$, which we
denote by $d_{\X}$, is a Riemannian metric.
Core to our methods are uniform samples of a manifold:

\begin{definition}[Uniformly Sampled Manifold Representation]\label{def:regsamp}
    Let $\X$ be a manifold with geodesics defined.
    The geodesic distance function on $\X$ is $d_{\X}: \X \times \X \to \R_{\geq 0}$
    where~$d_{\X}(x,y)$ denotes the geodesic
    distance between $x$ and $y$.
    For a uniform, finite sample $S \subset \X$,
    we can encode $S$ by pairwise distances
    within a distance matrix
    $D \in M^{|S|\times|S|}(\R_{\geq 0})$, with elements $d_{\X}(x,y) ~\forall x,y \in S$.
\end{definition}
Sometimes, instead of
having $D$ itself, we might have an approximation of it (based on neighborhood
graphs, for instance), denoted
$\widetilde{D}$ throughout this manuscript.
In this paper, properties of balls in $\R^n$ provide a natural backdrop to study more interesting spaces.
It is well known that in Euclidean space,
an open ball of radius $r$ under the standard Euclidean metric (denoted
$\ball{2}{x}{r}$) has \emph{volume}
\begin{equation}\label{eq:ballvol}
    \vol(\ball{2}{x}{r}) := \frac{\pi^{\frac{n}{2}}}{\Gamma(\frac{n}{2} + 1)} r^n ,
\end{equation}
where $\Gamma$ is the gamma function.
Let $\X$ be a Riemannian manifold.
We write the Gaussian curvature at a point $x \in \X$ as $\curv(x)$, and
the scalar curvature at $x \in \X$ as $\Sc(x)$.\camera{neither Gaussian
curvature nor scalar are actually defined. do we do this in the appendix? if so,
perhaps point there.}
Note that for two-manifolds, $\Sc(x) = 2\curv(x)$.
We often compare the volume of balls in $\R^n$ with balls on curved manifolds,
whose relationship is expressed by the following ratio:

\begin{theorem}[Relation Between Volume and Curvature]\label{thm:ratio}
    Assume $(\X, d_{\X})$ is an $n$-dimensional Riemannian manifold.
    Let $x \in \X$ be a point with scalar curvature~$\Sc(x)$.
    The ratio between the volume of the
    ball~$\ball{d_{\X}}{x}{r} \subseteq \X$ and
    the volume of the Euclidean ball~$\ball{2}{0}{r} \subset \R^n$ centered at zero
    is:
    \begin{equation*}
        \frac{\vol(\ball{d_{\X}}{x}{r})}{\vol(\ball{2}{0}{r})}
        = 1 - \frac{\Sc(x)\cdot r^2}{6(n+2)} + o(r^2).
    \end{equation*}
\end{theorem}
\thmref{ratio} allows us to inspect the relationship between volume and
curvature in Riemannian manifolds, and is proven in classical textbooks,
including~\cite[p.~168]{gallot2004} and~\cite[p.~317]{chavel}.
Likewise, the seminal Gauss-Bonnet theorem describes
the relationship between the total curvature of~$\X$ and its topology.

\begin{theorem}[Gauss-Bonnet\cite{lee2019intro}]\label{thm:gb}
    Let $\X$ be a compact Riemannian two-manifold, and let $x \in \X$. Then, the total curvature
    on $\X$ is $\int_{\X} \curv dA = 2 \pi \chi(\X)$,
    where $dA$ is the area element on $\X$ and $\chi(\X)$ is the Euler characteristic.
\end{theorem}
In higher dimensions, generalizations of the Gauss-Bonnet theorem become much more intricate and
are only defined for even dimensions.
Hence, in dimensions larger than two, we use properties of
the total mean curvature of submanifolds of $\R^n$, given in \cite{chen1984}.
In particular, we relate scalar
curvature, mean curvature, and the second fundamental form using the
Gauss-Codazzi equations for hypersurfaces; see \appendref{diff-eqn} for
definitions.

\begin{theorem}[Gauss-Codazzi Equations for Hypersurfaces \cite{haizhong1996}]\label{thm:gc-eqn}
    Let $\X$ be an $n$-dimensional Riemannian manifold immersed in $\R^m$.
    Then, the scalar curvature $\mathcal{S}$
    of $\X$ at a point $p \in \X$ is expressed by $\mathcal{S} = H^2 - ||h||^2$,
    where $H$ is the mean curvature of $\X$ at $p$, and $h$ is the second fundamental
    form of $\X$ at $p$; see \defref{lengthform}.
\end{theorem}

% --------------------------------------------------
% --------------------------------------------------
\subsection{Fundamental Definitions from (Spatial) Statistics} \label{subsec:statsBackground}

Let $(\X,d,\mu)$ be a metric measure space; that is, $\X$ is a topological space,
$S \subset \X$, 
$d \colon S \times S \to \R$ is a metric, and $\mu$ is a measure over a suitable
collection of subsets of $\X$. Then, let $X = \{X_i\}_{i=1}^m$ be $m$
samples drawn iid from the uniform distribution over $S$.
Let $A$ be a measurable subset of $S$, and, for each~$i$, let $A_i$ be the
random variable that is one if $X_i \in A$ and zero otherwise.  Since each $X_i$ is drawn iid from the
uniform distribution over $S$, we know that
in any realization of $X$, the expected number of points landing in $A$ is:
\begin{equation}\label{eq:expected-num-points}
    \E \left[ \#(X \cap A) \right] = \sum_{i=1}^m \E \left[ A_i \right] =
    \sum_{i=1}^m \P \left[ X_i \in A\right]= \frac{m \mu(A)}{\mu(S)}.
\end{equation}
Similar to the Buffon needle problem (e.g.,~\cite[pp.~71-2]{santalo}), we can
use this property to estimate geometric properties of $S$ (and of $A$) using the law
of large numbers.

\begin{example}[Estimating Areas]\label{ex:area-sampling}
    Consider $A \subseteq S \subset \R^2$, such that both $S$ and $A$ are Lebesgue-measurable.
    Then, for $m$ sufficiently large, the number of points that land in
    $A$ is
    approximately the ratio of the area of~$A$ to the area of $S$:
    \begin{equation}
        \left| A \cap X \right|
        \approx \frac{m \cdot  \area\left( A \right) }{ \area\left(\X \right)}.
    \end{equation}
    Multiplying both sides by $\vol(\X)$, we see that
    $\area(A) \approx \frac{m_A}{m}\cdot \area(\X)$, where $m_A$ denotes
    the number of sample points that land in $A$.  So, if we know $\area(S)$, we
    can use a realization of $X$ to estimate $\area(A)$. This generalizes to Reimannian manifolds.
\end{example}

We note here that this only holds if the points are sampled iid from the uniform
distribution over $S$.  If the points are sampled from some other distribution,
then this area estimation technique would not work for all measurable subspaces~$A$.
Additionally, the above construction is related to Ripley's
\ripkf, which assesses the homogeneity of point processes.
Let $S \subset \R^d$ be compact and Lebesgue-measurable, containing the subset
$S_R = \{x \in S| d_2(x, \partial S) \leq R \}$. In particular, our
formulation is related to the special case when a point process is drawn iid from a
uniform distribution, which results in the \ripkf $\K:[0,R] \to \R$ defined by
\[\K_{p}(r) = \vol(\ball{L_2}{p}{r}),\]
for a range of radii $[0,R]$ and a point $p$ sampled uniformly from $S \subset \R^d$.
The theoretical \ripkf is estimated with empirical versions; for more details and a rigorous definition
we refer readers to \appendref{spatial}.

% --------------------------------------------------
\subsection{Manifold Learning and Validation} \label{subsec:manifoldLearningBackground}

Having given the definitions relevant to manifolds themselves, we lay out
the paradigm on manifold learning and
its validation.
Broadly speaking, manifold learning operates on the assumption that data is sampled
uniformly from a manifold $\X$, and attempts to learn $\X$ by
approximating geodesic distances~\cite{Fefferman16a, Narayanan10a}.
Because of the variety in approaches, there is no agreed-upon definition of a manifold learning algorithm.
We give a general definition of \emph{manifold learning} that will be used throughout this work,
which is informed by the core structures of manifolds themselves.
That is, we consider arbitrary manifold
learning methods as a map from set of data to a distance matrix:

\begin{definition}[Manifold Learning]
    A \emph{manifold learning} model $F\colon S \to M^{|S|\times|S|}(\R_{\geq 0})$
    is a map from a set of data $S$ (assumed to be sampled uniformly from an $n$-dimensional
    ambient manifold) to an $|S|\times|S|$ real-valued distance matrix. We can think of $F$ as a map
    approximating geodesic distances, so we typically write $F(S) = \widetilde{D}$, and we write
    $D$ as the true geodesic distance matrix for all points of $S$.
\end{definition}

\begin{figure}
    \centering
    \begin{subfigure}{0.19\textwidth}
        \centering
        \includegraphics[width=0.99\textwidth]{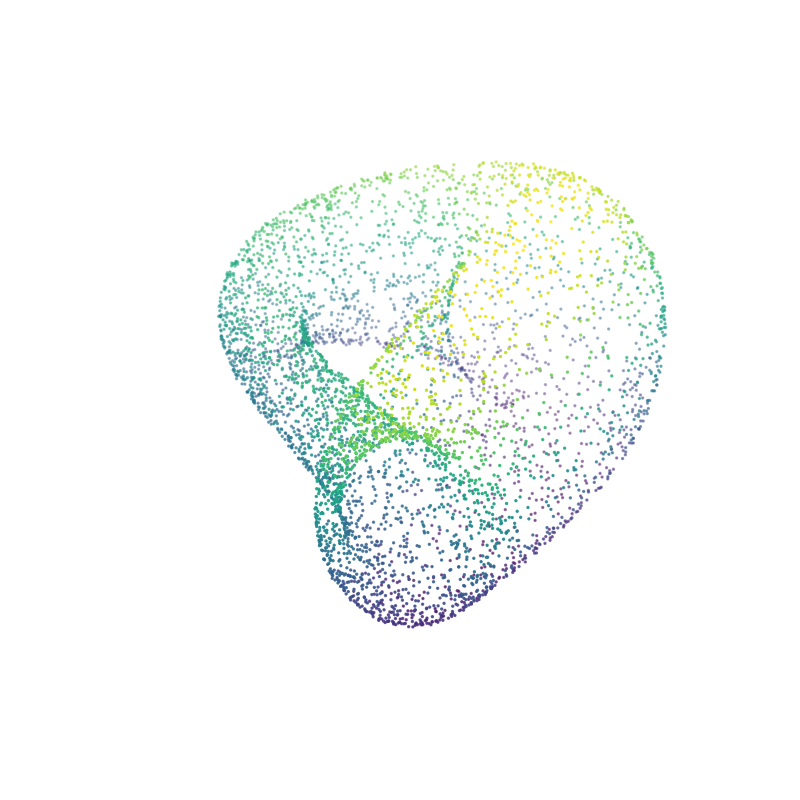} % second figure itself
    \end{subfigure}
    \begin{subfigure}{0.19\textwidth}
        \centering
        \includegraphics[width=0.99\textwidth]{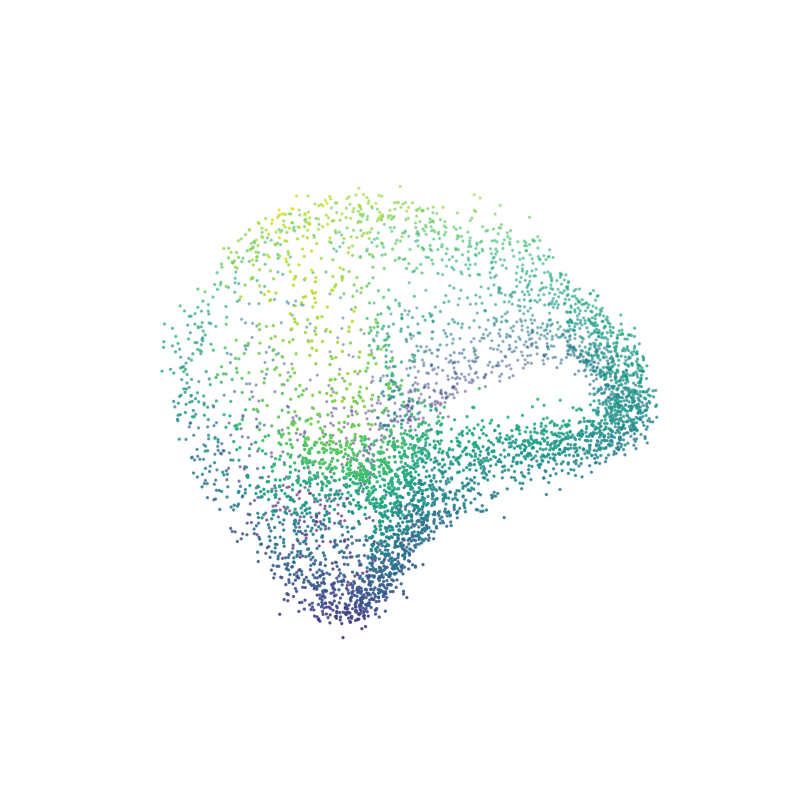} % first figure itself
    \end{subfigure}
    \begin{subfigure}{0.19\textwidth}
        \centering
        \includegraphics[width=0.99\textwidth]{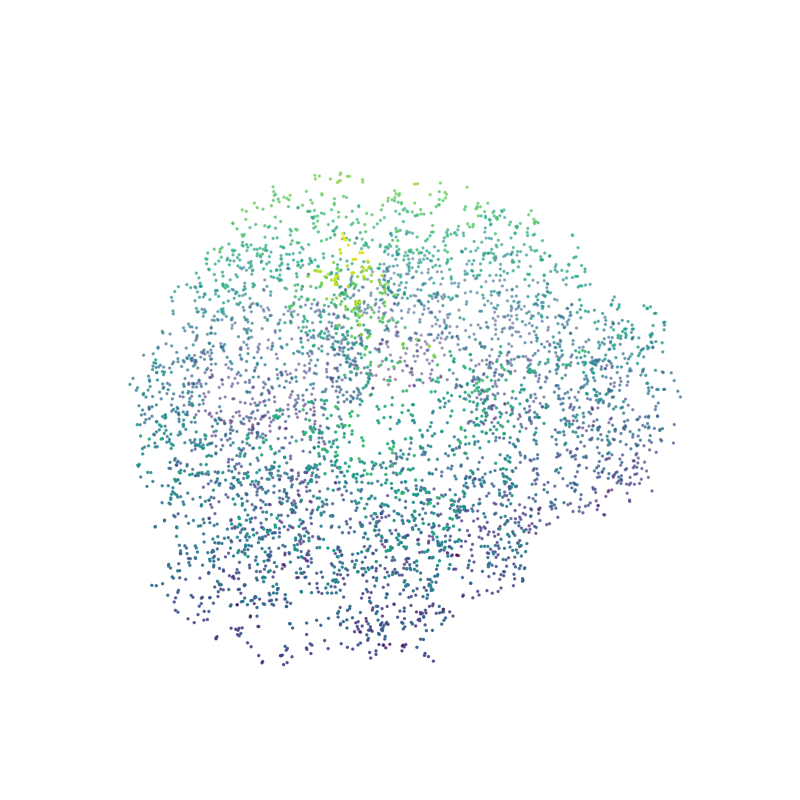} % second figure itself
    \end{subfigure}
    \begin{subfigure}{0.19\textwidth}
        \centering
        \includegraphics[width=0.99\textwidth]{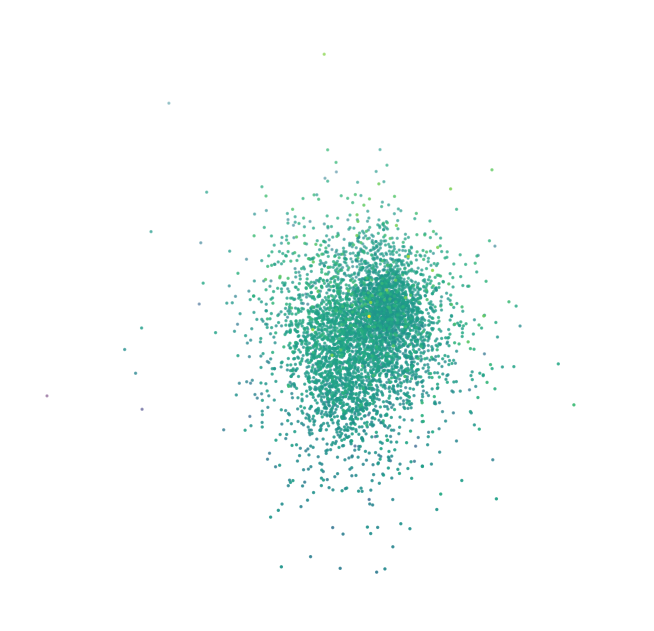} % first figure itself
    \end{subfigure}
    \begin{subfigure}{0.19\textwidth}
        \centering
        \includegraphics[width=0.99\textwidth]{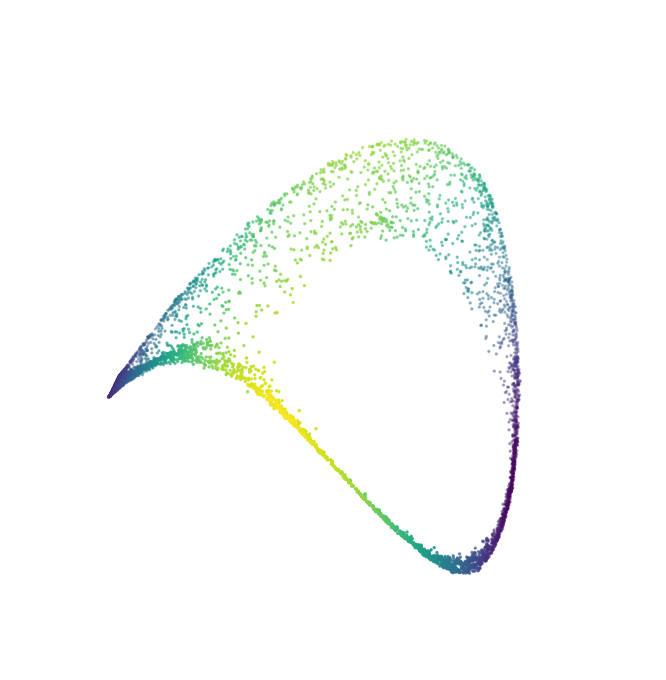} % first figure itself
    \end{subfigure}
    \caption{Constructing 3D embeddings of a Klein bottle lifted in ten dimensions, with varying success.
    From left to right, an output from PCA, ISOMAP, t-SNE, LLE, and spectral embedding. From visual inspection,
    PCA appears to have performed best in this example, followed by ISOMAP.
    }
    \label{fig:klein}
\end{figure}

It should be noted that many manifold learning
techniques compute an \emph{embedding} rather than a distance matrix, wherein geodesics are implicit.
In such cases, we can compute a geodesic distance matrix
by constructing a neighbor graph in the space of
the embedding and computing graph distances in a similar manner to the ISOMAP
algorithm of \cite{tenenbaum2000isomap}.
We refer the readers to \cite{huo2008, Zheng2009,hensel2021} for descriptions of
many different manifold learning techniques and frameworks.
Generally, our paper considers manifold learning in an
\emph{unsupervised} setting, assuming no knowledge
of the ground truth geodesic distances, $D$.
In this setting, validation must be done
using properties \emph{intrinsic} to the ``learned manifold'' approximated by $\widetilde{D}$ alone.
In general, intrinsic validation approaches have been quite successful in other areas of machine learning, for example
silhouette coefficients for clustering \cite{rousseeuw1987}, which rely only on properties intrinsic to clusters
themselves.

\section{An Intrinsic Manifold Validation Method} \label{sec:flatKFunc}
Let $(\X, d_{\X})$ be a Riemannian $n$-manifold.
Let $X \subset \X$ be a uniform sample on $\X$
By \thmref{ratio}, for $x\in \X$ and $r>0$,
we can write the volume of a ball~$\ball{2}{0}{r}$ in $\R^n$ as a function of
the scalar curvature at $x$ and
the volume of the ball~$\ball{d_{\X}}{x}{r}$ in $\X$:
\begin{equation}\label{eq:euclidballest}
    \vol(\ball{2}{0}{r})
    = \frac{\vol(\ball{d_{\X}}{x}{r})}
           {(1 - \frac{\Sc(x)\cdot r^2}{6(n+2)} + o(r^2))}.
\end{equation}
Furthermore, if $\vol{\X}$ is known, we can estimate
By \eqref{expected-num-points} and the law of large numbers, we know that if
$|X|$ is sufficiently large, then
the number of points landing in any measurable region $A \subset
\X$ is approximately $\frac{m \vol(A)}{\vol(\X)}$.  Thus, if we know $\vol(\X)$, we
can estimate $\vol(A)$, and vice versa.
Putting these two observations together, we define the \ourkf:
\begin{definition}[\ourkF And Its Estimators]\label{def:curve-disagg}
    Let $(\X, d_{\X})$ be a Riemannian~$n$-manifold, and
    let $X \subset \X$ be a uniform sample on $\X$. Let $p \in \X$, and let $R >
    0$ be sufficiently small.
    Then, the \emph{\ourkf} is the function
    $\fcntheo{\X} \colon [0,R] \to \R$
    defined by
    \begin{equation}
        \fcntheo{\X}(r)
        := \frac{\vol(\ball{2}{0}{r})}{\vol(\X)}
    \end{equation}
    with local estimator $\fcnest{p} \colon [0,R] \to \R$ defined by
    \begin{equation}
        \fcnest{p}(r)
        := \left(1 - \frac{\Sc(p)\cdot r^2}{6(n+2)}
        \right)^{-1} \cdot \frac{1}{|X|}\sum_{x \in X} I(x \in B_r(p))
    \end{equation}
    and an aggregated estimator $\fcnest{} \colon [0,R] \to \R$ defined by
    \begin{equation}
        \fcnest{}(r)
        := \frac{1}{|X|^2} \sum_{p \in X}
                \left( \left(1 - \frac{ \Sc(p) \cdot r^2}{6(n+2)}  \right)^{-1} \cdot
                \sum_{x \in X} I(x \in B_r(p)) \right)
    \end{equation}
\end{definition}

If $\X$ is flat, then $\Sc(p)=0$, simplifying $\fcnest{p}$ and $\fcnest{}$.
That is, we remove the first term of $\fcnest{p}$ and $\fcnest{}$, giving
$\fcnest{p}(r) := \frac{1}{|X|}\sum_{x \in X} I(x \in B_r(p))$ and
$\fcnest{}(r) := \frac{1}{|X|^2} \sum_{p \in X} \sum_{x \in X} I(x \in B_r(p))$.
Notice that the \ourkf is indeed related to Ripley's \ripkf in the special case for uniform
samples, in which case the two vary by only the factor $1/\text{Vol}(\X)$ in the theoretical
setting and $1/|X|$ in the empirical formulation.

\subsection{The Manifold Score}
We call the $L_2$-distance between a function and its estimate, for example,
$|| \fcntheo{\X} - \fcnest{} ||_2$, the \emph{error} of the estimate.
Since the estimators drop an $o(r^2)$ term from \eqref{euclidballest},
we note that $\fcnest{p}$ and $\fcnest{}$ are biased
estimators of $\fcntheo{}$; that is, for all $r \geq 0$, we have~$\fcnest{p}(r) \geq
\fcntheo{}(r)$ and there exists cases where equality doesn't hold.

Finally, we define a score
by taking a simple normalization of the error of
an estimator of $\fcntheo{\X}$.
\begin{definition}[Manifold Score]\label{def:score}
    Let $\X$ be a compact Riemannian manifold, and $X \subset \X$ a finite sample.
    The \emph{local manifold score} for $X$ as representing the manifold $\X$
    at~$p \in X$ is
    \begin{equation}
        \score{p}(X)
        = 1 - \frac{1}{|X|}|| \fcntheo{\X} - \fcnest{p} ||_2,
    \end{equation}
    and the \emph{aggregated manifold score} is
    \begin{equation}
        \score{}(X)
        = 1 - \frac{1}{|X|}|| \fcntheo{\X} - \fcnest{} ||_2.
    \end{equation}
\end{definition}
This score indicates the extent to which a sample resemples
a uniformaly distributed set of points  on a manifold.

\camera{
\brittany{I think the following comes out of the theorems in the next sections?}
\todo{somewhere note the normalization between
zero and one, with zero being ...}
Let $\X$ be a compact Riemannian manifold, and let $X \subset \X$
be a finite uniform sample on $\X$.
\bastian{%
  I think this point needs to be stressed more often; this is the basis
  of the whole approach, we should mention this more often in fact!%
}
As mentioned in \secref{intro}, intuitively one would expect that a ``good''
output $\dm{} = F(X)$
of a manifold learning model
$F$ should capture the uniformity of $X$, and a poor output might be very nonuniform.
Thus, for \emph{flat} manifolds, comparing an empirical and theoretical \ourkf using the
standard error
provides an intrinsic validation method for manifold learning.

We assume knowledge of only this distance matrix $\dm{}$ and its
dimension, and based on these two core properties of the data alone assess
how well the output $F$ resembles a manifold.
With a slight abuse of notation, we write both $p \in X$ or
$p \in \dm{}$ to denote a point in the sample $X$.
%
% \bastian{%
%   How are we deciding which one to use? Can we mention this earlier?%
% }
Given $\dm{}$, there are two directions: we can either compute the aggregated or
the local \ourkf.
The local and aggregated \ourkfs each have their own desirable properties.
In particular, any sample $X \subset \X$ admits a family of \ourkfs, and the aggregated setting allows
a \emph{single} summary of this family, whereas the local setting is informative of its full range and variance.
% We notate the local \ourkf on $\dm$ about a point $p \in \dm$
% % and denote the result as $\dkf$, or using the simpler notation from \defref{dkfn},
% as
% $\diskf$.
% Taking the average \ourkf over every $p \in \dm$ as in \defref{akfn}, we write the aggregated
% \ourkf as $\fcnest{}$.
For both, we compare functions using the
$L_2$-distance.\brittany{Maybe define this in the background?}
}
We note that \ourkf is well-defined and normalized, leaving verification
for the appendix. From this structure, a manifold score of one indicates a perfectly
uniform sample,
and a manifold score of one indicates a perfectly nonuniform sample, likely differentiating
between an effective and ineffective manifold-learning scheme.
\label{sec:score}
\subsection{\ourkFs For Manifolds With Curvature} \label{sec:kfuncWithCurv}
We now introduce a primary theoretical result of the paper, adapting
the \ourkf to manifolds with curvature. We begin with a brief discussion of
local \ourkfs and their shortcomings.
Related extensions for Ripley's $K$-function to more general settings have been
attempted before, and include extensions for spheres \cite{lawrence2018point, ward2021analysis}, 
and for fibers \cite{sporring2019}. We present the first extension of its kind
for general Riemannian two-manifolds and hypersurfaces, which in the aggregated setting
can be made entirely intrinsic.

At its core, our approach is designed with the specific intent
to assess the uniformity of a sample lying on arbitrary manifolds
rather than within Euclidean space alone. By its very formulation,
for an arbitrary point $p$ on a Riemannian manifold $\X$, the local \ourkf
is defined in terms of the scalar curvature $\Sc(p)$ at $p$. If we are equipped
with knowledge of $\Sc(p)$ at $p$, then we can compute the local \ourkf
directly. A straightforward but desirable consequence of this definition
is the following theorem, allowing us to understand the local \ourkf in terms of Euclidean balls:

\begin{theorem}\label{thm:dis}
    Let $\X$ be a compact Riemannian manifold, and let $X \subset \X$ be a uniform sample on $\X$. Let $p \in X$.
    Then, for sufficiently small $r>0$ and large $|X|$,
    $\diskf$ converges to the theoretical \ourkf $\fcntheo{\X}(r)$.
\end{theorem}
\begin{proof}
    Expanding definitions,
    \begin{align}
    \diskf &= \left(1 - \frac{\Sc(p)\cdot r^2}{6(n+2)} \right)^{-1} \cdot \frac{1}{|X|}\sum_{x \in X} I(x \in B_r(p)) \\
    &= \frac{\text{Vol}(\ball{2}{0}{r})}{\text{Vol}(\ball{\X}{r}{x})} \cdot \frac{1}{|X|}\sum_{x \in X} I(x \in B_r(p)) \text{ for sufficiently small $r$} \\
    &\to \frac{\text{Vol}(\ball{2}{0}{r})}{\text{Vol}(\ball{\X}{r}{x})} \cdot \frac{\text{Vol}(\ball{\X}{r}{x})}{\text{Vol}(\X)} \text{ by the law of large numbers} \\
    &= \frac{\text{Vol}(\ball{2}{0}{r})}{\text{Vol}(\X)} = \fcntheo{\X}(r).
    \end{align}

    Where we estimate area by sampling on Riemannian manifolds, see \exref{area-sampling}. It follows
    that $\diskf$ converges to $\fcntheo{\X}{(r)}$ for large sample size $|X|$, and sufficiently small $r$.
\end{proof}

The same derivation can be used for Ripley's \ripkf in the uniform case, dropping the $1/|X|$ term.
However, one should note that there are a few undesirable properties that go along with
\defref{curve-disagg} in the local setting. Most notably, our focus is on \emph{intrinsic} validation, and
knowing the scalar curvature at any point $p \in X$ violates the true spirit of an intrinsic
method. Additionally, \thmref{ratio} only technically holds for balls with sufficiently small
radius $r >0$, and the equation in general incurs an additive error term $o(r^2)$.
As such, the scaling in \defref{curve-disagg} becomes less accurate in approximating
$\fcntheo{\X}$ as $r$ increases. The latter problem can be mitigated by considering
small enough radii (although an exact bound on the error term remains a difficult and manifold-specific problem
in differential geometry),
and the former is resolved by considering global properties of the aggregated \ourkf.

\subsection{Aggregated \ourkFs for Surfaces}\label{sec:2d}
Let $\X$ a compact Riemannian two-manifold, and let $X \subset \X$ be a uniform sample of $X$.
In the aggregated setting, since $\fcnest{}$ is a global average of every \ourkf on $X$, we are able to avoid the
shortcomings in the local setting by exploiting fundamental results in differential geometry.
Namely, we first recall \thmref{gb}, the Gauss-Bonnet theorem, which
categorizes the relationship between geometry and topology for
two-manifolds, establishing that the total curvature of a manifold is a function of its Euler characteristic.
In addition, we recall that the scalar curvature at a point
$p \in \X$ on a two-manifold is twice the Gaussian curvature: $\curv(p) = 2\Sc(p)$.
Considering the total curvature over all of $\X$
and the ratio given in \thmref{ratio}, we can relate the total area of
balls of radius $r> 0$ in $\X$ to the total volume of
radius $r$ balls in $\R^2$ for sufficiently small $r$.

\begin{lemma}[Average Volume Distortion of Balls]\label{lem:avg}
    Let $\X$ be a compact Riemannian two-manifold. On average, the ratio
    $\frac{\text{Vol}(\ball{\X}{p}{r})}{\text{Vol}(\ball{2}{0}{r})}$ between the volume of $L_2$ and geodesic balls
    for a point $p \in \X$ is $$1 - \frac{\pi \chi(\X)}{24\cdot A} \cdot r^2 + o(r^2),$$ where $A$ is the area form of $\X$ and
    $\chi(\X)$ is the Euler characteristic.
\end{lemma}
\begin{proof}
    We expand definitions and integrate over the area form $dA$ of $\X$:
    \begin{align}
        \frac{1}{A} \int_{\X} \frac{\text{Vol}(\ball{d_{\X}}{x}{r})}{\text{Vol}(\ball{2}{0}{r})}dA 
        & = \frac{1}{A}\left[A - \int_{\X}\frac{\Sc(x)}{6(n+2)}\cdot r^2 dA + A o(r^2)\right] \text{ by \eqref{euclidballest}} \\
        & = 1 - \frac{1}{A} \int_{\X}\frac{\curv(x)}{48}\cdot r^2 dA + o(r^2) \\
        & = 1 - \frac{\pi \chi(\X)}{24\cdot A} \cdot r^2 + o(r^2),
    \end{align}
    giving a formula for the average ratio between the volume of Euclidean and geodesic balls in terms of the Euler characteristic,
    accompanied by an error term with value $o(r^2)$.
\end{proof}

This gives canonically an approximation dependent only on the Euler characteristic.
\begin{lemma}[Approximate Ratio for Surfaces]
    Let $\delta = \frac{r^2}{A}$ and $\epsilon = o(r^2)$, the additive error. Then we can estimate
    the left hand side of \lemref{avg} with $1 - \frac{\pi \chi(\X)}{24}$,
    which gives the approximation:
    \[\left|\frac{1}{A} \int_{\X} \frac{\text{Vol}(\ball{d_{\X}}{x}{r})}{\text{Vol}(\ball{2}{0}{r})}dA - 1 - \frac{\pi \chi(\X)}{24}\right| \leq \frac{(1-\delta)\chi(\X)}{24} + \epsilon.\]
\end{lemma}
\begin{proof}
    From \lemref{avg},
    $\frac{1}{A} \int_{\X} \frac{\text{Vol}(\ball{d_{\X}}{x}{r})}{\text{Vol}(\ball{2}{0}{r})}dA = 1 - \delta\frac{\pi \chi(\X)}{24}+ \epsilon$,
    giving
    
    \[\left| \left(1 - \delta\frac{\pi \chi(\X)}{24} - \epsilon\right) - \left(1 - \frac{\pi \chi(\X)}{24}\right) \right| \leq |(\delta-1)\frac{\pi \chi(\X)}{24}| + |-\epsilon|\]
    by the triangle inequality, which equals $(1-\delta)\frac{\pi \chi(\X)}{24} + \epsilon$ since $\delta \leq 1$ and $\epsilon > 0$.
\end{proof}
We thus have a definition
of total area distortion due to curvature on a two-manifold, which informs
our scaling procedure of the aggregated \ourkf on a two-manifold.
Namely, as a consequence of \lemref{avg}, the aggregated \ourkf on a general two-manifold
with curvature is \emph{invariant} of its Euler characteristic,
and converges to the standard \ourkf as the sample size increases:
\begin{theorem}\label{thm:agg-curved}
    Let $\X$ be a compact Riemannian two-manifold, and let $X \subset \X$ be a uniform sample on $\X$.
    For large $|X|$  and sufficiently small $r > 0$, $\fcnest{}(r)$ converges to
    $\fcntheo{\X}(r)$.
\end{theorem}
\begin{proof}
    We examine the result of \thmref{dis}
    when considered globally on $\X$.
    \begin{align*}
        \fcnest{}(r) &= \frac{1}{|X|^2} \sum_{p \in X} \left( \left(1 - \frac{ \Sc(p) \cdot r^2}{6(n+2)}  \right)^{-1} \cdot \sum_{x \in X} I(x \in B_r(p)) \right)\\
        &\to \frac{1}{A^2}\int_{\X}\frac{ \text{Vol}(\ball{d_\X}{p}{r}) dA}{(1 - \frac{\Sc(p)}{6(n+2)}*r^2)} \text{ for large $|X|$}\\
        &\to \frac{1}{A^2}\int_{\X} \vol(\ball{2}{0}{r}) dA \text{ for small enough $r$, by \eqref{euclidballest}} \\
        &= \frac{\vol(\ball{2}{0}{r})}{A^2} \cdot A = \text{Vol}(\ball{2}{0}{r})/A = \fcntheo{\X}(r).
    \end{align*}

    Where we integrate over the area form $A$ of $\X$. Moreover, the second to last equality
    follows from the fact that the volume of Euclidean balls $\vol(\ball{2}{0}{r})$
    is taken as a constant over $\X$ for a fixed radius $r$, and thus
    we integrate only over $dA$. Hence, for large $|X|$, $\fcnest{}(r) \to \fcntheo{\X}(r)$.
\end{proof}

Examining the integral in the denominator of Line 2 in
\thmref{agg-curved} gives a strong approximation of the aggregated \ourkf
$\fcnest{}$ on a surface, which depends on the global topology of $\X$ due to the
Gauss-Bonnet theorem and making use of the fact that $\mathcal{G}(p)
= 2\mathcal{S}(p)$ for surfaces:
\begin{corollary}\label{cor:agg-curved}
    Let $\X$ be a compact Riemannian manifold, and let $X \subset \X$ be a uniform sample on $\X$.
    We approximate:
    \begin{equation*}
        \fcnest{}(r)
        := \left(1 - \frac{ \pi \chi(\X) \cdot r^2}{24 A}  \right)^{-1} \cdot \frac{1}{|X|^2} \sum_{p \in X}
                \left( \sum_{x \in X} I(x \in B_r(p)) \right)
  \end{equation*}
    Then as $|X|$ increases, $\fcnest{}(r) \to \fcntheo{\X}(r)$. Moreover, taking $A = O(r^2)$, we have the approximation dependent only
    on the Euler characteristic:

    \begin{equation*}
        \fcnest{}(r)
        := \left(1 - \frac{ \pi \chi(\X)}{24}  \right)^{-1} \cdot \frac{1}{|X|^2} \sum_{p \in X}
                \left( \sum_{x \in X} I(x \in B_r(p)) \right)
  \end{equation*}

\end{corollary}

Consequently, in the aggregated setting we scale the average of local
\ourkfs by a function of the Euler characteristic
and the area of $\X$. In fact, as we demonstrate experimentally, the heuristic approximation of $\fcnest{}(r)$ given by taking $A = O(r^2)$
is typically sufficient, allowing us to scale the \ourkf using the Euler characteristic alone. Assuming no knowledge of the
Euler characteristic, we can simply estimate $\chi(\X)$ by selecting an
integer that minimizes $M(\fcnest{}(r))$.

\subsection{Aggregated \ourkFs in High Dimensions}\label{sec:highd}

As is indicated by the ratio given in \thmref{ratio}, the volume form of a Riemannian manifold $\X$
depends on the scalar curvature, and in two dimensions we are able to use the Gauss-Bonnet theorem
alone due to the direct relation between $\Sc(x)$ and $\curv(s)$ on surfaces.
Unfortunately, higher-dimensional generalizations of the Gauss-Bonnet theorem rely on
the Pfaffian of $\X$, which is a function of Ricci, Riemannian, and scalar curvature.
Instead,
we scale the aggregated \ourkf by
a different global constant of manifolds, namely~$\lambda_1$, the
first eigenvalue of their Laplacian operator.
We focus our attention where bounds to $\lambda_1$ are known, thus
dealing with orientable, compact hypersurfaces.
In particular, two useful bounds on $\lambda_1$ are given by the
following theorems, which relate $\lambda_1$ to the mean curvature $H$,
and to the squared length of the second fundamental form $h$. See
\appendref{diff-eqn} for more details.

\begin{theorem}[Total Mean Curvature \cite{chen1984}]\label{thm:mean}
    Let $\X$ be an $n$-dimensional, compact submanifold in $\R^m$, and $H$ the mean curvature at any point $p \in \X$. Then
    \begin{equation*}
    \int_{M}|H|^kdV \leq \left(\frac{\lambda_q}{n}\right)^{\frac{k}{2}}\text{Vol}(\X),
    \end{equation*}
    where $\lambda_q$ denotes the $q$th Eigenvalue of the Laplacian operator. 
\end{theorem}
In addition, the length of $h$ is related to $\lambda_1$ in the
following way:
\begin{theorem}[Total Length of the Second Fundamental Form \cite{chen1984}]\label{thm:length-form}
    Let $\X$ be a compact orientable $n$-dimensional submanifold of $\R^m$ with arbitrary codimension.
    Then,
    \begin{equation*}
    \int_{\X}||h||^2dV \geq \lambda_1\text{Vol}(\X).
    \end{equation*}
\end{theorem}
Equipped with these results, we have a mechanism to scale \ourkfs in
higher dimensions. Let $\X$ be a compact, $n$-dimensional orientable
manifold embedded in $\R^{n+1}$, i.e., $\X$ is a hypersurface in
Euclidean space.
Let $r>0$, where $r$ is sufficiently small to satisfy
the ratio given in \thmref{ratio}. As was done in \secref{2d}, we
compute the ratio of the total volume of balls of radius $r$ on $\X$ to
the total volume of balls of radius $r$ in $\R^n$. The following
integral is taken over the volume form $dV$ of $\X$, given formally in
\appendref{volume}:
\begin{align*}
    \int_{\X} \frac{\text{Vol}(\ball{d_{\X}}{x}{r})}{\text{Vol}(\ball{2}{0}{r})}dV &= \text{Vol}(\X) - \int_{\X}\frac{\Sc(x)}{6(n+2)}\cdot r^2 dV + \text{Vol}(\X)o(r^2) \\
    & = \text{Vol}(\X) - \int_{\X} \frac{H^2 - ||h||^2}{6(n+2)} \cdot
    r^2  dV + \text{Vol}(\X)o(r^2) \text{ (by \thmref{gc-eqn})} \\
    & = \text{Vol}(\X) - \frac{r^2}{6(n+2)}\left(\int_{\X}H^2 dV - \int_{\X} ||h||^2 dV \right) + \text{Vol}(\X)o(r^2). \\
\end{align*}
Now, notice from \thmref{mean} and \thmref{length-form}:
\[ \int_{\X}H^2 dV - \int_{\X} ||h||^2 dV \leq \frac{\lambda_1}{n}\text{Vol}(\X) - \lambda_1 \text{Vol}(\X). \]
This leads to the following approximation for the average ratio between geodesic and $L_2$ balls:
\begin{align*}
    & \frac{1}{\text{Vol}(\X)}\int_{\X} \frac{\text{Vol}(\ball{d_{\X}}{x}{r})}{\text{Vol}(\ball{2}{0}{r})}dV\\
    & \approx 1 - \frac{r^2}{6(n+2)\text{Vol}(\X)}\left(\frac{\lambda_1}{n}\text{Vol}(\X) - \lambda_1 \text{Vol}(\X) \right) + o(r^2) \\
    &= 1 - \left(\frac{r^2\cdot \lambda_1(1-n)}{6(n)(n+2)} \right) + o(r^2)
\end{align*}
More specifically, the above substitution is a $(2+1/n)$-approximation for the true value.
\begin{lemma}[Hypersurface Approximation Factor]\label{lem:hyper-approx}
    The difference
    $B = \int_{\X}H^2 dV - \int_{\X} ||h||^2 dV$ and its approximation $A = \frac{\lambda_1}{n}\text{Vol}(\X) - \lambda_1 \text{Vol}(\X)$
    satisfy $|B-A| \leq (2+1/n)|B|$.
\end{lemma}
\begin{proof}
    Both $\int_{\X} H^2dV, \int_{\X} ||h||^2 dV \geq 0$
    and $A = \frac{1-n\lambda_1}{n}\text{Vol}(\X) \leq 0$,
    implying $\frac{1}{n}\int_{\X} ||h||^2 dV \geq \int_{\X}H^2 dV$. This implies $|B| \leq  
    \int_{\X} ||h||^2 dV \leq \frac{\lambda_1}{n}\text{Vol}(\X) + |B|\leq (1+\frac{1}{n})|B|$.
    Moreover, since $|A| = |\frac{(1-n)\lambda_1}{n}\text{Vol}(\X)| \leq \lambda_1\text{Vol}(\X) \leq \int_{\X} ||h||^2 dV$,
    it follows  that $|B| + |-A| \leq (2+1/n)|B|$. By the triangle inequality, $|B-A| \leq (2+1/n)|B|$ as desired.
\end{proof}
Which obtains a desired analog for hypersurfaces scaling \ourkf
using $\lambda_1$.

\begin{definition}[Hypersurface Aggregated \ourkF] \label{def:agg-curve-highd}
    Let $\X$ a compact, orientable Riemannian $n$-manifold immersed in $\R^{n+1}$,
    and let $X \subset \X$ be a uniform sample of $\X$.
    The \emph{aggregated \ourkf for hypersurfaces} $\fcnest{}(r)$ is approximated by:
    \[\fcnest{}(r)
    \approx \left(1 - \frac{r^2\cdot \lambda_1(n-1)}{12(n)(n+2)}\right)^{-1} \cdot \frac{1}{|X|^2} \sum_{p \in X}
            \left( \sum_{x \in X} I(x \in B_r(p)) \right)\]
\end{definition}

It follows from \lemref{hyper-approx}
and \defref{curve-disagg} that \defref{agg-curve-highd}
attains a small multiplicative approximation factor, coupled with the standard additive error.

\begin{corollary}[Approximation Factor for \defref{agg-curve-highd}]
    Let $\fcnest{}(r)$ be as approximated above, then if $|X|$ is uniformly sampled from
    $\X$, we have $|\fcntheo{\X} - \fcnest{}(r)| \leq (2+1/n)\left(\frac{r^2\cdot \lambda_1(n-1)}{12(n)(n+2)}\right)^{-1} + \epsilon$
    where $\fcntheo{\X}$ is the theoretical \ourkf and $\epsilon = o(r^2)$ is the error of
    \thmref{ratio}.
\end{corollary}

Note that for many common hypersurfaces, $\lambda_1$ is either known, or
there exists a ``reasonable'' approximation; see
\appendref{eigenvalues} for examples.
Given any fundamental knowledge of the ambient
hypersurface, (e.g., if $\X$ can reasonably be assumed to lie on
a hypersphere, or similar manifold), we choose $\lambda_1$ explicitly
from the known approximations. Otherwise, analogously to tuning the
\ourkf by choosing the Euler characteristic $\chi(\X)$ minimizing
$M(\fcnest{})$, we choose the approximation for $\lambda_1$ minimizing
$M(\fcnest{})$.

\section{Desiderata}\label{sec:prop}
We remark on desirable properties when validating algorithms in machine
learning, and assess the computational complexity,
robustness, and intrinsic properties of our technique.

\subsection{Robustness}

It is pertinent when attempting the validation of random samples to
achieve stability against noise. Given a sample
$X \subset \X$ on a (flat or curved) Riemannian manifold $\X$, this
implies that subjecting a point $p \in X$ to small perturbation $\delta
> 0$ should change the validation score of $M(p)$ very little.
We thus examine the local and aggregated manifold scores
$M(\fcnest{p})$ and $M(\fcntheo{\X})$ with respect to their stability
after subjecting points of
$X$ to noise.
In the aggregated setting, it is not difficult to make robustness
guarantees about the \ourkf.
Specifically, for a point $p \in X$, examine $p \pm \delta$.
Indeed, for a fixed radius $r >0$, in the worst case, $\diskf$ could
change at most by $|X|$, and for every other $p' \not=p \in \X$, $\hat
K_{p'}(r)$ could change by at most $1$.
This alters the aggregated \ourkf, $\fcntheo{\X}$, by at most $\pm 2/|X|$.
However, a small change by $\delta$ to $r$ gives the identical \ourkf
prior to the perturbation, i.e., $\hat K_{p}(r\pm\delta) = \hat K_{p}(r)$.
The worst-case bound remains the same if we allow perturbations of every
$p \in \X$ by $0 \leq \delta' \leq \delta$ for both $\fcntheo{\X}$ and $\diskf$
when considering a fixed radius $r$. Hence, we can consider the
aggregated \ourkf to be resistant to noise for fixed radii, and both the
local and aggregated  \ourkfs can be considered resistant to
noise when allowing a change in radius $r \pm \delta$. Consequently, our
manifold score $M(\fcntheo{\X})$ is robust to noise up to the equivalent
thresholds.

\ben{might want to formalize this a bit more}

\subsection{Computation}

Let $X \subset \X$, where $\X$ is an $n$-dimensional Riemannian
manifold with $|X| = m$.
Using our construction, adjusting the computation of the
aggregated \ourkf to arbitrary manifolds only adds a constant factor in
runtime complexity under reasonable assumptions about the complexity of a manifold
and the dimension.
For surfaces, assuming that $\chi(\X)$ is a fixed constant, we could have
$O(1)$ \ourkf estimations, each of
which adds only an $O(1)$ factor in finding the value for $\chi(\X)$ minimizing
$M(\fcntheo{\X})$. For hypersurfaces, again we search over $O(1)$ known
formulas for $\lambda_1$, each of which scale computation of $\fcntheo{\X}$ by an $O(1)$
factor. Assuming the dimension is $O(d)$, we can use the same procedure
to naively compute the dimension minimizing $M(\fcntheo{\X})$, adding
only an $O(d)$ factor to the complexity.
Given an algorithm to compute Ripley's $K$-function, we can compute \ourkf
adding only an $O(1)$ factor, when the values of $d, \chi(\X)$,
and $\lambda_1$ are constant.
Naively, computing the \ourkf in either the aggregated or disaggregated settings are computed
in $\Theta(n^2)$ time. More efficient
implementations are likely possible by efficiently embedding the given distance matrix in $\R^n$,
and then using an $n$-dimensional range tree
to compute $\diskf$ for a given radius $r$. Doing so is possible in $O(\log^{n-1}m + k)$
time using fractional cascading, where $k = \diskf$. Finding $\fcntheo{\X}$ overall is done by computing
$\diskf$ for every $p \in X$, which takes
$O(m (\log ^{n-1} m + k_{max}))$ time,
where $k_{max}$ is the largest
number of points in any ball of radius $r$ throughout computation of
the \ourkf. If taken throughout the entire range $R$ of radii, this is
$O(Rm\log^nm)$ time.
In practice, optimized implementations for Ripley's $K$-function exist for a
variety of applications, such as \cite{tang2015, streib2011,
sporring2019, hohl2019}.

\subsection{Intrinsic Properties}

We conclude the section with a summary of the intrinsic properties of the \ourkf
in the aggregated setting. Specifically, there are two quantities needed to adopt the scaling
needed to accurately compute the \ourkf from a sample on a Riemannian manifold $X \subset \X$,
and these are (1) the dimension, and (2) the scalar curvature
of points $p \in X$. A primary contribution of this work is showing how to subvert knowledge
of the scalar curvature in the aggregated setting using either $\chi(\X)$ or $\lambda_1$. Indeed,
if $\X$ is a two-manifold, for sufficiently small radii the error term $o(r^2)$ vanishes, and so long as
$r^2 = O(A)$, there will exist a range of $r$ minimizing $M(\fcnest{p})$ for the true Euler characteristic
and dimension. The same behavior is true for hypersurfaces with (known) eigenvalue formulas differing by
at least a factor of three, due to \lemref{hyper-approx}.

\section{Discussion}\label{sec:discussion}
In this paper, we introduce the \ourkf to build a robust,
efficiently computable, intrinsic framework for manifold validation.
Our approach takes inspiration from Ripley's \ripkf, and extends
the \ripkf in the case for uniform samples.
We prove convergence properties and bound the accuracy when
approximating the \ourkf on two-manifolds and hypersurfaces.
In \secref{experiment}, we also provide an implementation, experimentally
verifying the included results.
Further extensions to this work include consideration of broader classes
of manifolds and surfaces,
and refining the technique for manifolds with boundary. 
% Extending these
% results to wider categories of objects will expand the utility of this framework. 
Additionally, to further understand the use of this tool in applied
settings, we wish to test the results of this framework in a wide
array of datasets popular in applied manifold learning, especially biochemical
data.

\paragraph{Acknowledgements}
B.H.\ acknowledges support from the U.S. Department of Energy under Grant Number DE-SC0024386.
E.Q., J.S., and B.T.F.\ acknowledge support from NSF under Grant Number 1664858.
E.Q.\ and B.T.F.\ acknowledge support from NSF under Grant Numbers 1854336 and 2046730.
E.Q.\ also acknowledges support from the Montana State University Undergraduate Scholars Program.
B.R.\ is supported by the Bavarian state government with
funds from the \emph{Hightech Agenda Bavaria}.

\newpage
\bibliographystyle{siam}
\bibliography{main}

\appendix
% \crefalias{section}{appendix} % uncomment if you are using cleveref

\section{Additional Properties of the \ourkf}
For the sake of completeness, we prove some basic but important properties
of \ourkf, and the associated error from the theoretical \ourkf.
\begin{lemma}[Range of Local \ourkFs]\label{lem:range}
    Let $\X$ be a compact Riemannian manifold, let $X \subset \X$ a finite uniform sample,
    and fix a point $p \in X$.
    For the local \ourkf, we have $\diskf \in [0, 1]$
    for each radius $r$.
\end{lemma}
\begin{proof}
    If $r=0$, then $\ball{\X}{p}{0}=\emptyset$.
    Thus, by definition,
    $\diskf = 0$.
    Additionally, by the compactness of $\X$, there exists $\maxR \geq 0$ such
    that~$\ball{\X}{p}{\maxR}$ covers all of
    $X$. This implies $\sum_{x \in X} I(x \in \ball{\X}{x}{\maxR}) = |X|$.
    Increasing $\maxR$ by $\delta > 0$
    continues to have $\ball{\X}{p}{\maxR + \delta}$ covering all of $X$,
    so the maximum value attained by $\diskf$ is $|X|/|X| = 1$.
\end{proof}
The same result holds for aggregated \ourkf.
\begin{lemma}[Range of Aggregated \ourkF]\label{lem:range2}
    Let $\X$ be a compact Riemannian manifold, let $X \subset \X$ a finite uniform sample,
    and fix a point $p \in X$.
    For the aggregated \ourkf, we have $\fcnest{} \in [0, 1]$
    for each radius $r$.
\end{lemma}
\begin{proof}
    Let $p,x \in X$. Again, the minimum value attained by $\fcnest{}$ is vacuously $0$, by \\
    setting $r=0$.
    By \lemref{range}, the range of $\diskf$ is $[0,1]$ for every $p \in X$.
    Denote $\maxR^x$ as the minimum radius needed
    such that~$\ball{\X}{x}{\maxR^x}$ covers all of $X$.
    Next, find the $x \in X$ such that $\maxR^x \geq \maxR^p$ for any other
    $p\not=x \in X$. Then $\hat{K}_{p}(\maxR^x) = 1$ for every $p \in X$, so we
    have $\hat{K}_{\dm{X}}(\maxR^x) = \frac{|X|\cdot 1}{|X|} = 1$.
    Hence, the range of $\fcnest{}$ is again $[0, 1]$.
\end{proof}

With the possible range of $\fcnest{p}$ and $\fcnest{}$ established, we now bound the
error of any $\fcnest{p}$ and $\fcnest{}$ with respect to the theoretical \ourkf, $\fcntheo{\X}$.

\begin{lemma}[Bounding Error]\label{lem:range3}
    Let $\X$ be a compact Riemannian manifold, and $X \subset \X$ a finite uniform sample.
    The maximum error of $\fcnest{}$ is $|X|$.
\end{lemma}
\begin{proof}
    Both of $\fcnest{p}$ and $\fcntheo{\X}$ are proportions defined on the interval, $[0,R]$.
    Maximally, $R$ could be $|X|$, and the error could be $1*|X|$.
\end{proof}

From \lemref{range}, \lemref{range2}, and \lemref{range3} we guarantee the well-definedness and normalization
of $M$. Hence, \defref{score} quantifies how ``manifold-like''
a distance matrix $\dm{X}$ is on the interval $[0,1]$.
Intuitively, a manifold score of zero indicates that an uncovered sample from
some manifold learning technique was completely nonuniform (with all
points of $\dm{X}$ coinciding), and therefore the manifold learning model can be thought to have performed poorly.
On the other hand, a manifold score of one says that an uncovered sample was perfectly uniform, that is,
identical to the theoretical \ourkf.
By assumption, since $\dm{X}$ was an approximation of a uniform sample $X$ on
an ambient manifold $\X$, we consider  $\dm{X}$ to be a good approximation
in this case.

\section{Ripley's $K$-Function}\label{append:spatial}
Here, we provide a few concepts from spatial statistics, but refer the reader
to \cite{cressie2015statistics}, \cite{chiu2013stochastic},
\cite{daley2003introduction} and \cite{karr2017point} for a more comprehensive
overview of the relevant statistical definitions. Given $n \in \N$ and a Borel
space~$S$ with measure $\mu$, a \emph{(binomial) point process} $\Phi$ of $n$
points is a random variable valued in functions $f \colon S \to \N$,
where~$||f||_{1}:= \sum_{s \in S} f(s) =n$.
In other words, one way to think of $\Phi$ is as a random variable valued in in
finite multisets (i.e., points) of size $n$ in~$S$; here, usually, $S$ is
a compact subspace of $\R^d$, but, more generally, we allow $S$ to be
a Borel-measurable metric space~\cite{ripley1976locally}. In this light, we
write $\Phi$ as the sum of $n$ Dirac delta functions:
$\Phi = \sum \delta_{X_i}$,
where~$X_i$ is a random variable valued in $S$.

One common tool to study spatial point patterns, such as a set of points
sampled from a manifold, is Ripley's \ripkf~\cite{kfcn}. In particular, the
\ripkf is a tool to assess complete spatial randomness (homogeneity) of a point
process. As we saw in \subsecref{statsBackground}, if the sample is truly
uniformly distributed, then the proportion of the sample points that land
within a specific region is proportional to the volume of that region.

\begin{definition}[Ripley's $K$-Functions]\label{def:dkfn}
    Let $n \in \N$, $R \geq 0$, and $\Phi = \sum_{i=1}^n \delta_{X_i}$ be
    a binomial point process with $n$ points over a Lebesgue-measurable,
    compact set $S \subset \R^d$.
    Define the set~$S_{R} := \{ x \in S ~|~ d_2(x,\bdry S) \leq R \}$, where
    $\bdry S$ is the boundary of $S$. For each
    $p \in S_R$, \emph{Ripley's local \ripkf} for~$\Phi$ at~$p$ is the
    function~$\K_p \colon [0,R] \to \R$ defined by\footnote{More generally, the
    $K$-function can be defined for all points in $D$; however, care must be
    taken when considering points near the boundary of $D$.}
    \begin{equation}
        \K_{p}(r) = \frac{\vol(S)}{n}
            \E \big[ \left| \ball{d}{p}{r} \cap \{ X_i \}_{i=1}^n\right| \big],
    \end{equation}
    where $\E[\cdot]$ denotes expectation.

    The \emph{aggregated $K$-function} is the function $\K \colon [0,R] \to \R$
    defined by
    \begin{equation}
        \K(r) := \int_{p \in S} K_p(r) \, d\mu(p)
    \end{equation}
\end{definition}

In short, Ripley's \ripkf for a random sample is proportional to the expected
number of points that land in~$\ball{d}{p}{r}$. If the points are taken iid
from the uniform distribution over $S \subset \R^2$, then, by \eqref{expected-num-points},
we have~$\E\left[ \K_{p}(r) \right]= \frac{\vol(S)}{n} \cdot \frac{n\vol(\ball{d}{p}{r})}{\vol(S)} = \pi r^2$,
and as stated in \subsecref{statsBackground} gives in general the volume of balls in $\R^d$.

Let $X$ be a realization of $\Phi$. To estimate~$K_p$ in practice, we use the
empirical $K$-function, which is the function~$\estK_p \colon [0,R] \to \R$
defined~by:
\begin{equation}\label{eqn:local-empirical-kfun}
    \estK_{p}(r)
    = \widehat{\lambda}^{-1} \left[ \sum_{x \in X} I(x \in \ball{2}{p}{r})
    \right],
\end{equation}
where $I$ is the indicator function and $\widehat{\lambda}$ is an estimator of
density per unit area ($n \slash \vol(S)$). Often, it is more practical to
divide this quantity by $\vol(S)$, resulting in $\widetilde{\K}_{p}(r)=
\frac{1}{\vol{S}} \cdot \estK_{p}(r)$, representing the proportion of sample
points that land in $\ball{2}{p}{r}$. The \emph{empirical local~$K$-function}
is the collection of functions~$\{ \widehat{K}_x \}_{x \in X}$.
Analogously, the aggregated empirical~\ripkf~$\estK \colon [0,R] \to \R$ is
defined by:
\[\estK(r) = \frac{1}{n} \sum_{x \in X} \estK_x(r).\]

Note that these definitions assume homogeneity of the point process, and hence
$\lambda$ is a constant. If, however, this assumption is not met, we can adjust
our estimate in Equation \ref{eqn:local-empirical-kfun} by weighting points in
the point process by the estimated intensity measure at that point (see e.g.
\cite{baddeley2000non}):

\begin{equation}\label{eqn:local-empirical-kfun-inhom}
    \estK_{p}(r)
    =  \left[ \sum_{x \in X} \widehat{\lambda}(x_i)^{-1} I(x \in \ball{2}{p}{r})
    \right],
\end{equation}

Also, note that the definition we present for the K-function is
biased-corrected using the boundary method (see \cite{ripley1988statistical} Ch. 3)
by only considering the points $S_R$ that is the set S eroded by distance $R$.
See Figure \ref{fig:point-realizations} for a visualization of this for both
a homogeneous and an inhomogeneous point process. Other bias correction methods
exist to account for this bias. For example, one method weights points near the
boundary by the proportion of $\ball{d}{p}{r}$ that lies within the manifold.

\begin{figure}
    \includegraphics[width=\linewidth]{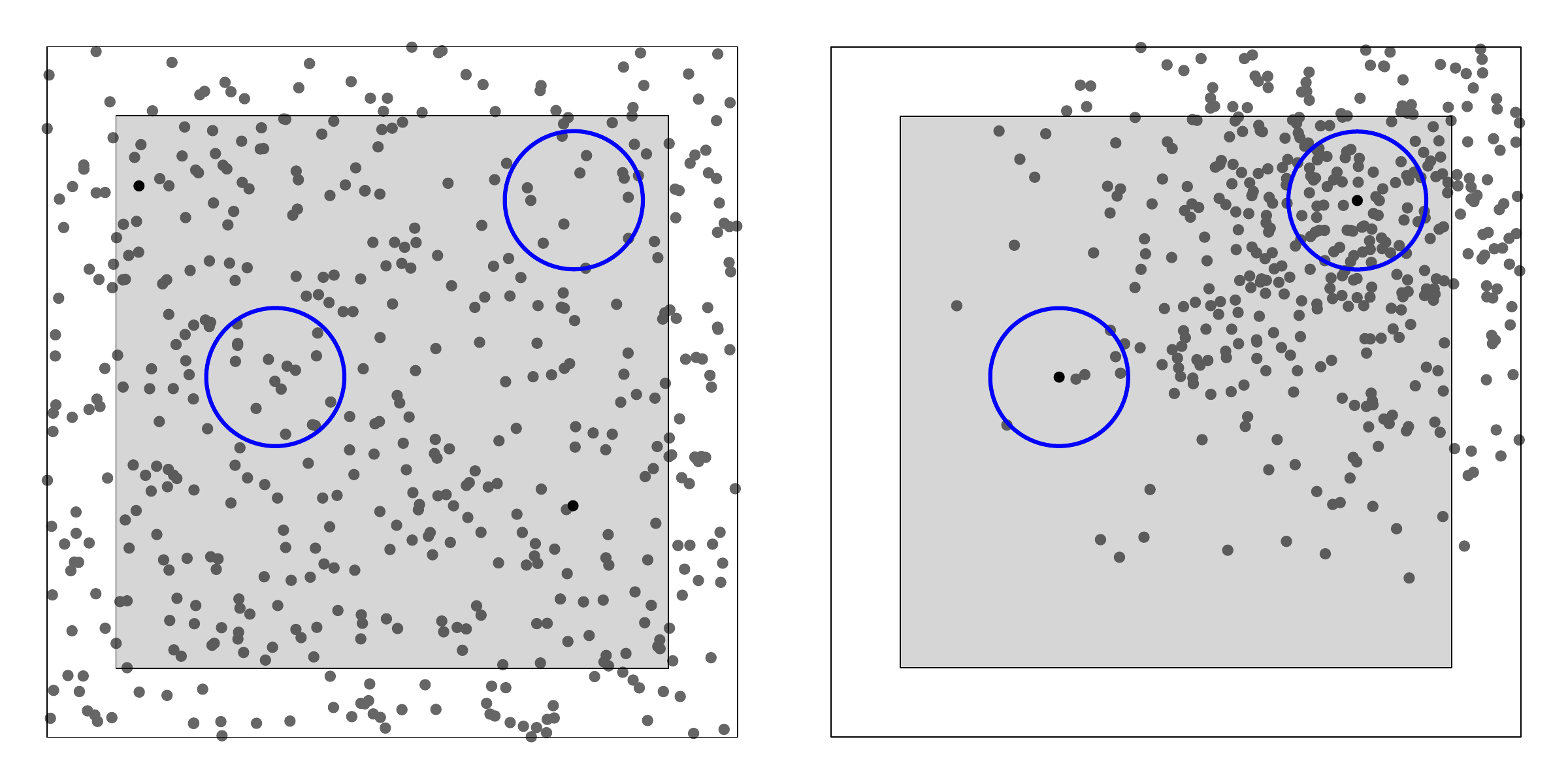}
    \caption{Two realizations of a Binomial point processes. Left shows a
    homogeneous point process and right shows an inhomogeneous point
    process.\camera{change circles to disks}}
    \label{fig:point-realizations}
\end{figure}

\section{Definitions from Differential Geometry}
\subsection{A Discussion of Intrinsic and Extrinsic Validation}
In general, \emph{model validation} takes a model and quantifies the correctness
of its output.
There are two main classes of validation algorithms:
\emph{extrinsic} and \emph{intrinsic}.

In \emph{extrinsic} validation, the labels in data are known.
In this case, the validation process compares a
model's output to the ground truth in order to
quantify the model's correctness.
Extrinsic validation is generally the easiest and strongest method for validation,
one particularly ubiquitous variant is precision and recall
which has been studied in the context for manifolds in \cite{geodesicPrecisionRecall}.

In \emph{intrinsic} validation, the ground truth
correct output of a model is not known.
Thus, the validation process must use properties
of only the model's output in order to
evaluate its correctness.
A particularly notable example of intrinsic validation includes
silhouette coefficients for clustering \cite{rousseeuw1987}.
In practice, extrinsic validation typically accompanies
supervised learning and intrinsic validation
is often coupled with unsupervised learning.
As the title alludes, this paper
focuses on intrinsic validation.

% --------------------------------------------------

\subsection{A Review of Basic Definitions in Differential Geometry}\label{append:def-diff}

We present fundamental definitions from differential geometry in the following section.
For a survey, we suggest \cite{chen1984} and \cite{lee2019intro}.

We begin with a formal definition of Riemannian metrics, which are crucial in the definition of
a Riemannian manifold
\begin{definition}[Riemannian metric]
    Let $(\X,d)$ be a smooth manifold. We call $d$ a \emph{Riemannian metric} if $d$ is a smooth covariant
    two-tensor field whose value $d_x$ at each $x \in \X$ is an inner product on the tangent space $T_x\X$.
\end{definition}

On a manifold $\X$, we often study paths. A
\emph{path} between $x, y \in \X$ is a continuous map $\gamma:[0,1]\to\X$,
where $\gamma(0) = x$ and $\gamma(1) = y$. A path $\gamma$ has \emph{length},
which is defined as follows:

\begin{definition}[Length]\label{def:length}
    Let $(\X, d)$ be a topological space and let $\gamma$ be a \emph{path} in $(\X, d)$.
    Let~$\mathcal{P}$ be the set of all finite subsets $P= \{t_i\}$  of $[0,1]$ such that
    such that~$0 = t_0 < t_1 < ... < t_n = 1$. The
    \emph{length} $L_d(\gamma)$
    of~$\gamma$ is:
    \[ L_d(\gamma) := \sup_{P \in \mathcal{P}} \sum_{i=1}^{n}{d(\gamma(t_i), \gamma(t_{i-1}))}. \]
\end{definition}

The shortest paths between elements in $\X$ are usually of particular interest,
which are called \emph{geodesics}.

\begin{definition}[Geodesic]\label{def:geodesic}
    Let $(\X, d)$ be a topological space, let $x, y \in \X$, and let $\Gamma$
    be the set of all paths in $(\X, d)$ from $x$ to $y$.
    A \emph{geodesic} $\gamma^*$ from $x$ to $y$ is a path with shortest length:
    \[\gamma^* = \inf_{\gamma \in \Gamma} L_d(\gamma)\]
\end{definition}

Let $U \subseteq \X$ be an open subset of $\X$, and
let $\phi:U \to \R^n$ be a coordinate chart of $\X$ (that is, $\phi$ is a homeomorphism).
We say that $\phi$ is \emph{smooth}
if for every $u \in U$, every component function of $\phi$ has
continuous partial derivatives of every order.
Let $U, V \subseteq \X$ be open subsets of $\X$ and examine
two coordinate charts $\phi:U\to\R^n$ and $\psi:V \to \R^n$.
We call $\phi, \psi$ \emph{smoothly compatible} if both $\phi \circ \psi^{-1}$
and $\psi \circ \phi^{-1}$ are smooth when defined on $\phi(U \cap V)$ and $\psi(U \cap V$)
respectively. An \emph{atlas} of $\X$ is a set of coordinate charts whose domains cover $\X$,
and an atlas is a \emph{smooth atlas} if any two charts in the atlas are smoothly compatible.
A \emph{smooth structure} is a smooth atlas not properly contained in another smooth atlas.
Finally, a \emph{smooth manifold} is a manifold $\X$ equipped with a smooth structure.

Let $x \in \X$ and let $U$ be an open neighborhood of $x$. Let $C^{\infty}(V,\R)$
be the set of all smooth mappings $f:U\to\R$. Let $f, g \in C^{\infty}(V,\R)$.
A \emph{tangent vector} $T_x$ on $\X$ at $x$
is a linear map $T_x:C^{\infty}(V,\R) \to \R$ satisfying the Leibniz law of products:
\[T_x(f\cdot g) = f(x)T_x(g) + g(x)T_x(f).\]

The \emph{tangent space} on $T_x\X$ on $\X$ at $x$ is the set of all tangent vectors on $\X$ at $x$.

\subsection{A Review of Tensor Products and the Volume Form} \label{append:volume}

We begin with a condensed review of tensors, and describe
one of their most important constructions in differential geometry: the volume form.
For a more detailed description, we refer
readers to \cite{lee2019intro}. We begin with alternating tensors and wedge products,
and then proceed to the volume form itself.

\begin{definition}[Alternating Tensor]
    Let $V$ be a finite-dimensional vector space, and let $F$ be a
    covariant $k$-tensor on $V$. We call $F$ an \emph{alternating tensor}
    if the interchanging of two arguments causes $F$ to change sign:
    \[F(v_1, ..., v_i, ..., v_j, ... , v_k) = -F(v_1, ..., v_j, ..., v_i, ... , v_k)\]
\end{definition}

If $F$ is an alternating $k$-tensor, then it follows that for an arbitrary permutation of arguments $\sigma$,

\[F(v_{\sigma(1)}, ... , v_{\sigma(k)}) = \text{sgn }\sigma F(v_1, ... , v_k).\]

Where sgn $\sigma$ is the sign of $\sigma$, giving $+1$ if $\sigma$ is an even permutation
and $-1$ if $\sigma$ is odd. Similarly, the \emph{alternation} of $F$ is given as follows:

\begin{definition}[Alternation]
    Let $F$ be a covariant $k$-tensor. Let $S_k$ be the set of permutations of $F$.
    The \emph{alternation} Alt $F$ of $F$ is given by the expression:
    \[\text{Alt $F$}(v_1, ... v_k) = \frac{1}{k!}\sum_{\sigma \in S_k} (\text{sgn }\sigma) F(v_{\sigma(1)}, ..., v_{\sigma(k)})\]
\end{definition}

It is easy to check that Alt $F = F$ iff $F$ is alternating.
Using the Alt operation, we define the wedge product of tensors:

\begin{definition}[Wedge Product]
    Let $V$ be a finite-dimensional vector space.
    Let $F$ be an alternating covariant $k$-tensor on $V$, and let $G$ be an alternating
    covariant $l$-tensor on $V$. Then, their \emph{wedge product} is defined by:
    \[F \wedge G = \frac{(k+l)!}{k!l!}\text{Alt}(F \otimes G)\]

    Where $F \otimes G$ is the usual tensor product.
\end{definition}

Now, in order to understand the volume form rigorously, we present its definition below.
Many definitions give the same property as the one below, along with two equivalent
expressions. For simplicity we provide only the first property given in most definitions.

\begin{definition}[Volume Form]
    Let $(\X,g)$ be an oriented Riemannian manifold (with or without boundary).
    Let $(\eps_1, ..., \eps_n)$ be a local oriented orthonormal coframe for $T^*\X$, the cotangent bundle
    of $\X$. The \emph{volume form}
    $dV_g$, or $dV$ as a shorthand, is the unique $n$-form on $\X$ such that $dV_g = \eps_1 \wedge ... \wedge \eps_n$.
\end{definition}

By design, if we have a continuous function $f:\X \to \R$ on a compact orientable Riemannian manifold (with or without boundary),
then $f \cdot dV_g$ is a compactly supported $n$-frame, and taking an integral $\int_{\X}f \cdot dV_g$ is well defined.
Moreover, if $\X$ is compact, the \emph{volume} of $\X$ is defined canonically:
\[\text{Vol}(\X) = \int_{\X}dV_g.\]

\subsection{The Gauss-Codazzi Equations for Hypersurfaces}\label{append:diff-eqn}

As with basic tensor operations, this article assumes a some working knowledge of concepts of curvature
in differential geometry. For a detailed review of Riemannian, Ricci, Scalar, and Mean curvatures,
as well as the second fundamental form,
we again refer readers to Lee \cite{lee2019intro}.
\ben{probably need more defs here from diff geom}

In the results for hypersurfaces, we use the \emph{second fundamental form} and the \emph{length}
of the second fundamental form: 

\begin{definition}[Second Fundamental Form]\label{def:form}
    Let $M$ be an $n$-dimensional hypersurface in of a Riemannian manifold 
    $M'$, and let $D$ and $D'$ be their covariant derivatives. Then, for
    $m \in M$ and $u,v \in T_m M$ the second fundamental form $h$
    of $M$ is given by: 
    \[h(u,v) := \left(D'_U V - D_u V \right)_m'\]

    where $U,V$ are vector fields tangent to $M$ at any point of $M$, defined in a neighborhood of $m'$ in $M'$,
    and with respective values $u$ and $v$ at $m$.
\end{definition}

\begin{definition}[Length of Second Fundamental Form]\label{def:lengthform}
    Let $h$ be the second fundamental form at a point $p \in \X$, a Riemannian manifold.
    The squared \emph{length} $||h||^2$
    of $h$ is given by:
    \[||h||^2 = \sum_{i,j=1}^n h(E_i, E_j)^2\]

    Where $E_1, ..., E_n$ is a local orthonormal frame of tangent vector fields over $\X$.
\end{definition}

\subsection{Hypersurfaces and their First Laplacian Eigenvalue}\label{append:eigenvalues}

Naturally, after showing that the aggregated \ourkf can be scaled as a function of $\lambda_1$ in high dimensions,
one wonders how to actually compute $\lambda_1$. The following lays out the bulk of known results for $\lambda_1$
with respect to a number of different manifolds.

\begin{theorem}[Conformal Clifford Torus]
    If $M$ is a conformal Clifford torus, Then
    \[\lambda_1 * vol(M) \leq 4*\pi^2\]
\end{theorem}

\begin{theorem}[Veronese Surface]
    If $M$ is a conformal Veronese surface, then we have
    \[\lambda_1 * \text{vol}(M) \leq 12*\pi,\]

    with equality holding if and only if $M$ admits an order 1
    isometric embedding.
\end{theorem}

\begin{theorem}[Submanifolds of Hyperspheres]
    Let $M$ be an $n$-dimensional compact submanifold of a hypersphere $S^m(r)$ of radius $r$ in $\R^{m+1}$.
    Then $\lambda_1 \leq n$, with equality holding iff $M$ is of order 1.
\end{theorem}

\begin{theorem}[Projective Plane]
    Let $M$ be a compact, $n$-dimensional, minimal submanifold of $\R P^m$, where $\R P^m$ is of constant sectional
    curvature $1$. Then it follows that
    $\lambda_1 \leq 2(n+1)$, with
    equality holding iff $M$ is a totally geodesic $\R P^n$ in $\R P^m$.
\end{theorem}

\begin{theorem}[Complex Projective Plane]
    Let $M$ be an $n$-dimensional $(n \geq 2)$, compact, minimal submanifold of $\mathbb{C} P^m$, where $\mathbb{C} P^m$
    is of constant holomorphic sectional curvature 4. Then we have
    $\lambda_1 \leq 2(n+2)$

    with equality holding iff $n$ is even, $M$ is a $\mathbb{C} P^{\frac{n}{2}}$, and $M$ is a complex totally
    geodesic submanifold of $\mathbb{C} P^m$.
\end{theorem}

\begin{theorem}[Quaternion Projective Plane]
    Let $M$ be a compact, $n$-dimensional $(n \geq 4)$, minimal submanifold of $\mathbb{Q} P^m$,
    where $\mathbb{Q} P^m$ is of constant quaternion sectional curvature 4. Then we have
    $\lambda_1 \leq 2(n+4)$,
    with equality holding iff $n$ is a multiple of $4$, $M$ is $\mathbb{Q} P^{\frac{n}{4}}$,
    and $M$ is embedded in $\mathbb{Q} P^m$ as a totally geodesic quaternionic submanifold.
\end{theorem}

\begin{theorem}[Cayley Plane]
    Let $M$ be a compact, $n$-dimensional, minimal submanifold of the Cayley Plane $OP^2$, where $OP^2$
    is of maximal sectional curvature 4. Then we have
    $\lambda_1 \leq 4n$.
\end{theorem}

\begin{theorem}[CR Submanifolds of $\mathbb{C}P^n$]
    Let $M$ be a compact, $n$-dimensional, minimal, $CR$-submanifold of $\mathbb{C} P^m$. Then
    we have
    $\lambda_1 \leq 2(n^2 + n + 2a)/n$, where $a$ is the complex dimension of the holomorphic distribution.
\end{theorem}

\begin{theorem}[CR Submanifolds of $QP^m$]
    Let $M$ be a compact, $n$-dimensional, minimal CR-submanifold of $QP^m$. Then we have
    $\lambda_1 \leq 2(n^2+n+12a)/n$
    where $a$ is the quaternionic dimension of the quaternion distribution.
\end{theorem}

\section{Experimental Results}\label{sec:experiments}
  We present our experimental findings, comparing the manifold score of
  samples on popular manifolds.
  Our anonymized code is accessible at the following url for reproducibility:
  
\href{https://anonymous.4open.science/r/intrinsic-manifold-validation/README.md}{https://anonymous.4open.science/r/intrinsic-manifold-validation/}

  \subsection{Results on Flat Manifolds}

  As an initial verification, we test the framework on
  an embedding of the flat $2$-torus,
  where we expect the manifold score to be nearly perfect (close to one) for
  a uniform sample.
  We compare this against a stratification on the flat torus, sampled uniformly in the form of a ``cross,''
  $\{ (x,y) : x = \frac{1}{2} \text{ or } y = \frac{1}{2} \} \subset [0,1)\times[0,1)$.
  We also consider a sample of the ``cross''-stratification with noise, where ten percent
  of the sample's points are uniformly sampled from the domain graph of the flat torus.
  Table \ref{exp:flat-torus} gives the average aggregated manifold score
  across ten trials on the flat $2$-torus for these three samples.
  As anticipated, the manifold score is nearly perfect for the uniform sample, worse
  for the stratification with noise, and even worse for the pure stratification.

  \begin{table}[h]
    \caption{Manifold Score on the Flat Torus} \label{exp:flat-torus}
    \begin{center}
    \begin{tabular}{llll}
    \textbf{Sample Size}  &\textbf{Uniform Sample} &\textbf{Strat. with Noise} &\textbf{Stratification} \\
    \hline \\
    100         & $0.9840 \pm 0.0056$ & $0.8807 \pm 0.0158$ & $0.6768 \pm 0.0490$\\
    200             & $0.9921 \pm 0.0023$& $0.8769 \pm 0.0178$ & $0.6410 \pm 0.0339 $ \\
    500             & $0.9962 \pm 0.0010$ & $0.8763 \pm 0.0071$ & $0.6537 \pm 0.0210 $  \\
    1000             & $0.9980 \pm 0.0009$ & $0.8728 \pm 0.0074$ & $0.6545 \pm 0.0241 $ \\
    \end{tabular}
    \end{center}
  \end{table}

  \subsection{Surfaces with Curvature}

  \begin{figure}
    \centering
    \begin{subfigure}{0.24\textwidth}
        \centering
        \includegraphics[width=0.99\textwidth]{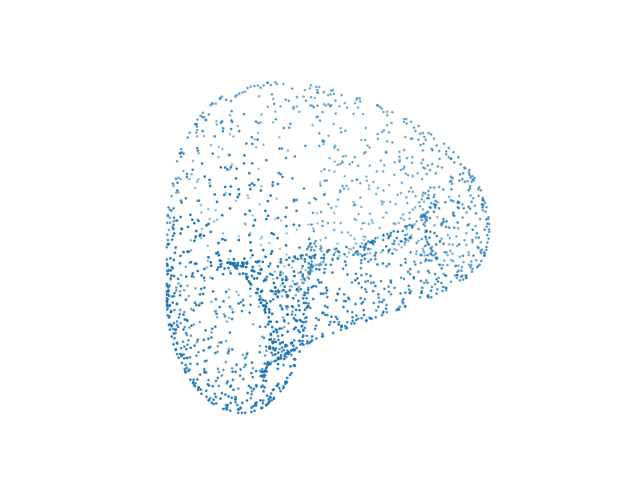} % first figure itself
    \end{subfigure}
    \begin{subfigure}{0.24\textwidth}
        \centering
        \includegraphics[width=0.99\textwidth]{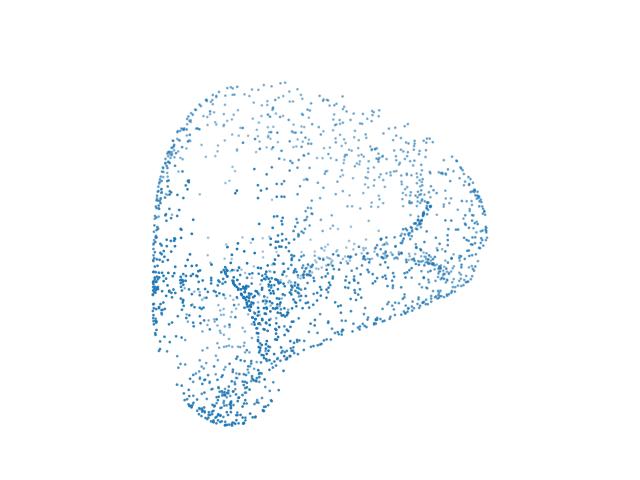} % second figure itself
    \end{subfigure}
  \begin{subfigure}{0.24\textwidth}
      \centering
      \includegraphics[width=0.99\textwidth]{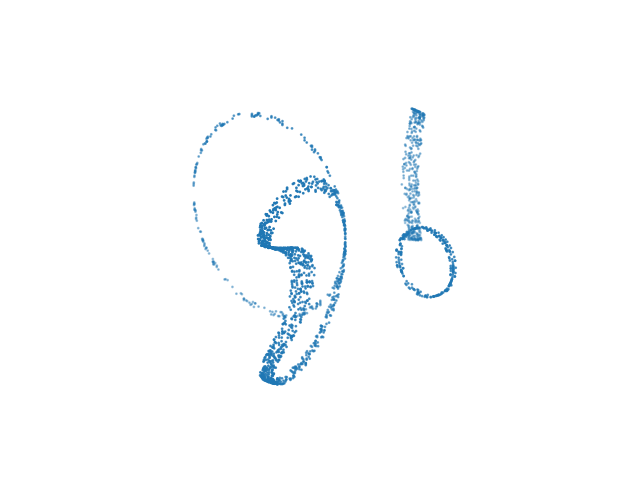} % first figure itself
  \end{subfigure}
  \begin{subfigure}{0.24\textwidth}
      \centering
      \includegraphics[width=0.99\textwidth]{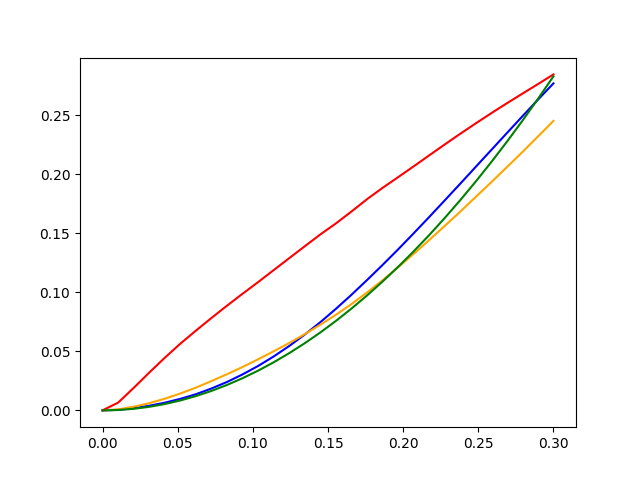} % second figure itself
  \end{subfigure}
  \caption{From left to right: (1) uniform sample vs. (2) minor noise vs (3) stratification on the Klein bottle with $|X|=2000$. The
  right plot compares the normalized \ourkfs of each, with (1) in blue, (2) in
      orange, (3) red, and the theoretical \ourkf in green.}
\end{figure}

  We compare the manifold score of stratifications
  vs. uniform samples on both the sphere and the Klein bottle.
  We generate a distance matrix using ISOMAP with six neighbors,
  conducting ten trials for each experiment. 
  For the sphere, we scale according to the approximation
  following \corref{agg-curved} setting $A \approx r^2$.
  A uniform sampling and ``cross''-stratification were considered, and 
  the results are presented in Table \ref{tbl:sphere}. 
  We note the impact of proper scaling on the manifold score.
  For the Klein bottle, we use no scaling since the Euler characteristic is zero.
  We again consider a uniform sample and ``cross''-stratification,
  as well as a noisy sample obtained by sampling with probabilities
  defined by a normalized sine wave across the surface.
  As demonstrated by Table \ref{tbl:klein}, the manifold score is effective 
  in evaluating the sample's representation of the Klein bottle. 

  For both surfaces,
  the manifold score converges to one (a perfect score)
  as the uniform sample size increases, 
  which is explained by \thmref{dis}.
  However, this convergence is slower than for the 
  flat torus,
  which we attribute to the introduction of a neighbors graph
  rather than using true geodesics, as
  well as the heuristic nature of the definition of the
  aggregated \ourkf for general two-manifolds.

  \begin{table}[H]
    \caption{Results on the Sphere} \label{tbl:sphere}
    \begin{center}
    \begin{tabular}{llll}
    \textbf{Sample Size}  &\textbf{Scaled Unif.} &\textbf{Unscaled Unif.} &\textbf{Stratification} \\
    \hline \\
    100         & $0.9415 \pm 0.0169$ & $0.7770 \pm 0.0147$ & $0.6116 \pm 0.1801$ \\
    500             & $0.9443 \pm 0.0083$ & $0.7834 \pm 0.0256$ & $0.6101 \pm 0.1842$ \\
    1000             & $0.9583 \pm 0.0151$ & $0.7840 \pm 0.0123$ & $0.6597 \pm 0.2110$ \\
    2000             & $0.9730 \pm 0.0096$ & $0.7963 \pm 0.0110$ & $0.6310 \pm 0.2382$ \\
    \end{tabular}
    \end{center}
  \end{table}

  \begin{table}[H]
    \caption{Results on the Klein Bottle} \label{tbl:klein}
    \begin{center}
    \begin{tabular}{llll}
    \textbf{Sample Size}  &\textbf{Uniform} & \textbf{Minor Noise} &\textbf{Stratification} \\
    \hline \\
    200             & $0.8834 \pm 0.0571$ & $0.8114 \pm 0.0760$ & $0.3578 \pm 0.0926$ \\
    500             & $0.9109 \pm 0.0334$ & $0.8942 \pm 0.0386$ & $0.4048 \pm 0.1521$ \\
    1000             & $0.9430 \pm 0.0105$ & $0.8417 \pm 0.1286$ & $0.5054 \pm 0.0558$ \\
    2000             & $0.9529 \pm 0.0081$ & $0.8735 \pm 0.0520$ & $0.3075 \pm 0.0630$ \\
    \end{tabular}
    \end{center}
  \end{table}

\begin{figure}
  \centering
  \begin{subfigure}{0.31\textwidth}
        \centering
        \includegraphics[width=0.99\textwidth]{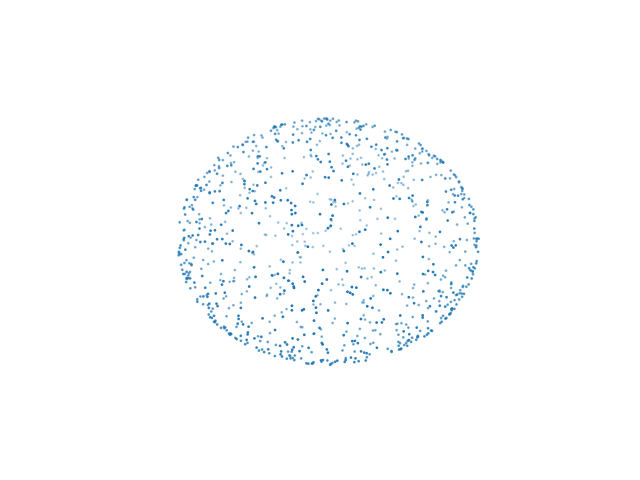} % first figure itself
    \end{subfigure}
    \begin{subfigure}{0.31\textwidth}
        \centering
        \includegraphics[width=0.99\textwidth]{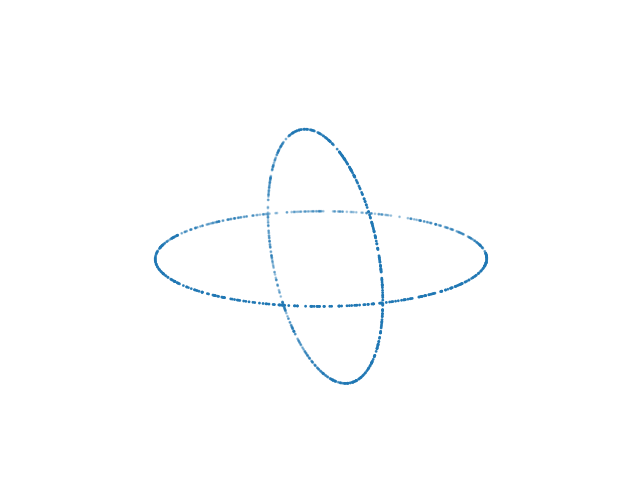} % second figure itself
  \end{subfigure}
  \begin{subfigure}{0.31\textwidth}
      \centering
      \includegraphics[width=0.99\textwidth]{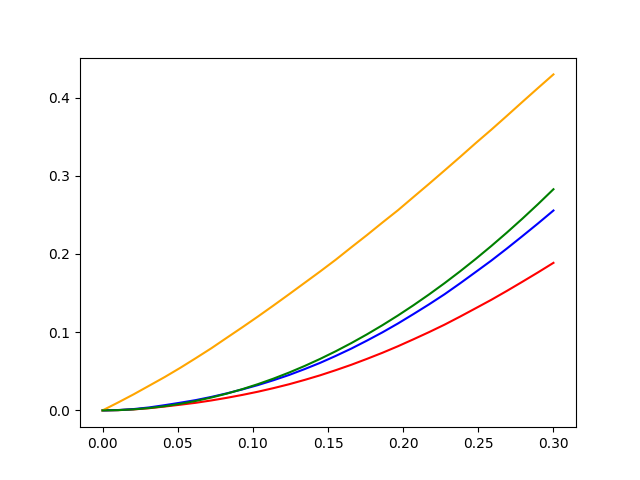} % first figure itself
  \end{subfigure}
  \caption{From left to right: (1) uniform sample vs. (2) a stratification on the sphere, with $|X|=1000$. The
      right plot compares the \ourkfs of each, with (1)
      scaled in blue, (1) unscaled in red, (2) in orange, and the theoretical
      \ourkf in green.}\label{fig:spherestrat}
\end{figure}

  Next, we develop a more in depth discussion of the manifold score on the Klein bottle, and implement
  the comparison of manifold learning algorithms shown in \figref{klein} in \secref{prelim}.
  This anecdote provides a sense for both the intended positive behavior and the shortcomings
  of our use of the \ourkf, which appear inherent to intrinsic validation in general.
  Namely, we lift a parameterization of the Klein bottle in $\R^3$ to ten dimensions by
  assigning values from the interval $[0,2\pi]$ uniformly at random to each of the other
  seven coordinates. In total, this point cloud representation lying on an ambient Klein
  bottle in 3D has 1000 points. 
  We then attempt to re-learn the 3D parameterization of the
  Klein bottle; computing embeddings in $\R^3$ of the point cloud in $\R^10$ using PCA, ISOMAP,
  t-SNE, LLE, and spectral embedding, each with $n_{components}=3$ and $n_{neighbors}=6$.
  To the naked eye, PCA appears to perform the best, simply by dropping the
  other seven dimensions, and ISOMAP also appears to perform well, albeit a bit worse. In actuality,
  ISOMAP assigns much greater density to the handle region of the Klein bottle in this example,
  and achieves a somewhat unfavorable score. This provides an adversarial example where t-SNE outperforms ISOMAP, despite the visual
  preferability of the latter.
  The results when evaluating the manifold score on each algorithm are summarized in the following table,
  where we report the manifold score on the \ourkf restricted to radius $r=0.3$:

  \begin{table}[H]
    \caption{Differing Manifold Scores from the Klein Bottle}
    \begin{center}
    \begin{tabular}{lllll}
    \textbf{PCA}  &\textbf{ISOMAP} & \textbf{t-SNE} & \textbf{LLE} & \textbf{Spectral Embedding} \\
    \hline \\
    $0.9557$ & $0.7179$ & $0.9154$ & $0.3988$ & $0.6881$
    \end{tabular}
    \end{center}
\end{table}

\begin{figure}
  \centering
  \begin{subfigure}{0.31\textwidth}
      \centering
      \includegraphics[width=0.99\textwidth]{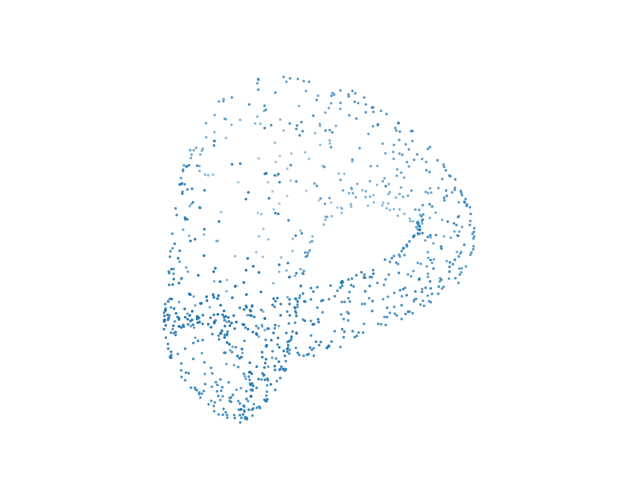}
  \end{subfigure}
  \begin{subfigure}{0.31\textwidth}
      \centering
      \includegraphics[width=0.99\textwidth]{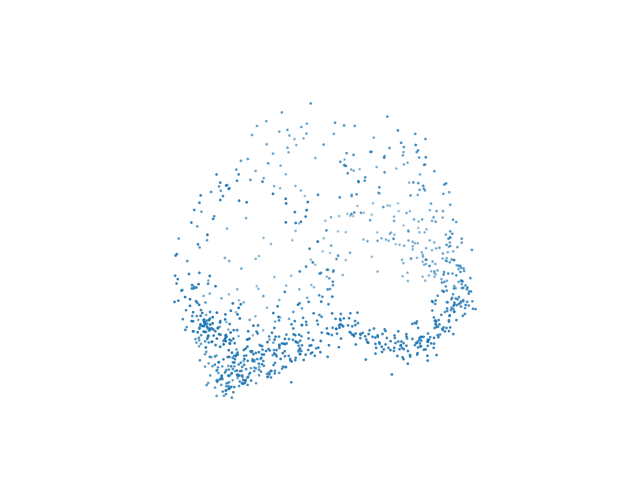}
  \end{subfigure}
\begin{subfigure}{0.31\textwidth}
    \centering
    \includegraphics[width=0.99\textwidth]{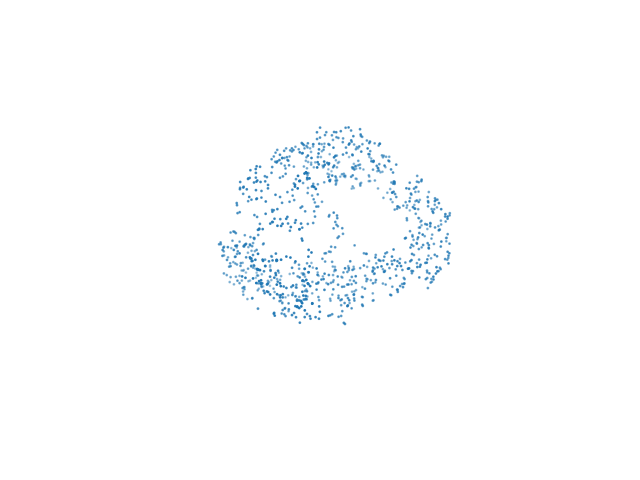}
\end{subfigure}
\begin{subfigure}{0.31\textwidth}
  \centering
  \includegraphics[width=0.99\textwidth]{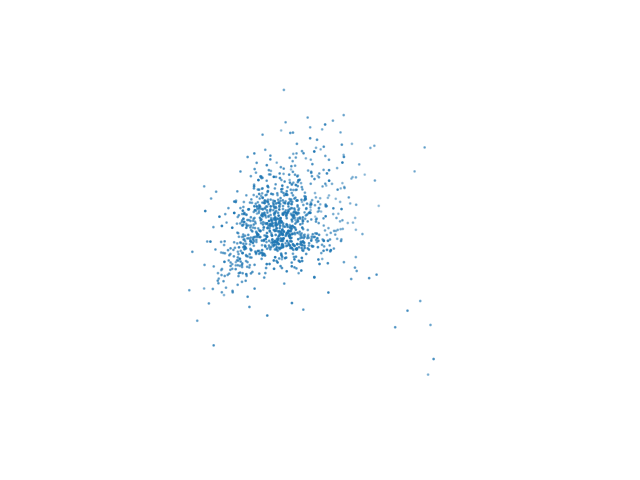}
\end{subfigure}
\begin{subfigure}{0.31\textwidth}
  \centering
  \includegraphics[width=0.99\textwidth]{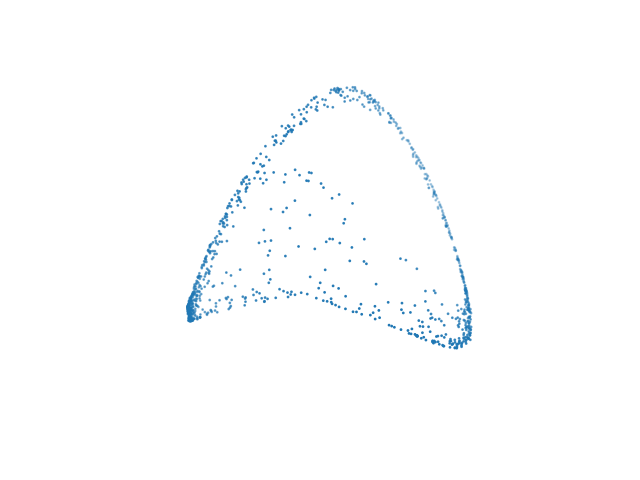}
\end{subfigure}
\caption{Embeddings of the Klein bottle in $\R^3$ after being projected from $\R^{10}$ from left to right: PCA, ISOMAP, and t-SNE (top)
locally linear embedding and spectral embedding (bottom).}
\end{figure}

\begin{figure}[H]
  \centering
  \includegraphics[width=0.9\textwidth]{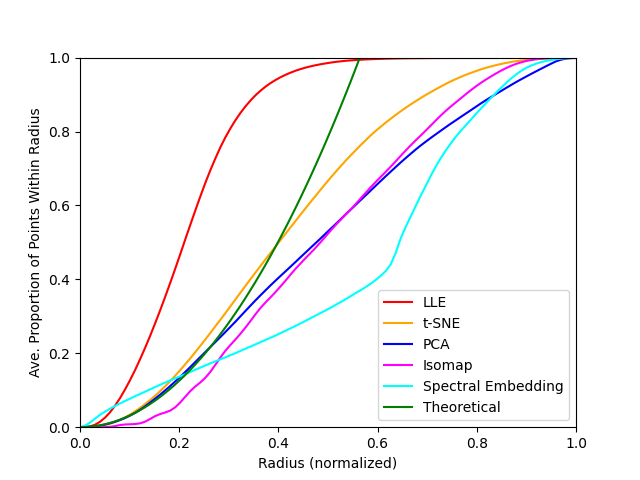}
\end{figure}

Upon visual inspection, the apparent best performing algorithm in this example (PCA) is indeed detected as performing well,
while the worse performing algorithms are generally detected as such. However, the manifold score is not able to distinguish between
ISOMAP and t-SNE in the example, despite visually ISOMAP appearing to better uncover the manifold.
This illuminates a shortcoming of the manifold score; which is a consequence of working in the unsupervised setting.
That is, without knowledge of the ground truth manifold which we are attempting to uncover,
there is no way of distinguishing between a uniform sample of something more distantly resembling a desired manifold,
and a nonuniform sample closely resembling a desired manifold.

  \subsection{Hypersurfaces}

  In higher dimensions, we use the results of \secref{highd}
  to examine uniform samples on hyperspheres
  in six, eight, ten, and fifteen dimensions. We demonstrate the manifold score of a uniform sample
  pre and post scaling, and the analogous stratification to the one in \figref{spherestrat} of two $d/2$-dimensional hyperspheres (scaled equivalently).
  For each experiment, we work with samples of size 1000 over ten trials. In higher dimensions the volume of Euclidean
  balls decreases, making the manifold score increasingly sensitive; yet we are still able to delineate between
  uniform samples and stratifications with the needed specificity.
  % We provide the complete experiment and analysis for hypersurfaces in \appendref{exp}.

  \begin{table}[h]
    \caption{Results on Hyperspheres of Differing Dimension} \label{exp:klein}
    \begin{center}
    \begin{tabular}{llll}
    \textbf{Dimension}  &\textbf{Scaled Unif.} & \textbf{Unscaled Unif.} & \textbf{Stratification} \\
    \hline \\
    6             & $0.9747 \pm 0.0022$ & $0.9143 \pm 0.0056$ & $0.8033 \pm 0.0023$ \\
    8             & $0.9803 \pm 0.0010$ & $0.9256 \pm 0.0034$ & $0.8782 \pm 0.0043$ \\
    10             & $0.9895 \pm 0.0004$ & $0.9458 \pm 0.0019$ & $0.9239 \pm 0.0029$ \\
    15             & $0.9968 \pm 0.0003$ & $0.9692 \pm 0.0033$ & $0.9697 \pm 0.0014$ \\
    \end{tabular}
    \end{center}
  \end{table}

In what follows, we outline the resulting \ourkfs for hyperspheres,
tabulated in their corresponding dimension. We provide a more complete experiment here,
including edge cases in lower and higher dimensions than the cases included in experiments
in the main body. Interestingly, only upon reaching five dimensions does scaling
the \ourkf for hypersurfaces become beneficial, which we attribute to the fact that
our scaling is indeed approximate; if possible to work with two-manifolds the exact scaling
afforded by the Gauss-Bonnet theorem is preferable.
In very high dimensions, as expected, it becomes increasingly subtle (but still possible
outside of the range of error)
to distinguish between uniform samples nonuniform samples such as the stratification. Once again,
we employ a stratification of two $d/2$ hyperspheres sampled on the surface of the $d$-dimensional
hypersphere.

\begin{table}[h]
    \caption{Results on Hyperspheres of Differing Dimension} \label{exp:klein}
    \begin{center}
    \begin{tabular}{llll}
    \textbf{Dimension}  &\textbf{Post Scaling} & \textbf{Pre Scaling} & \textbf{Stratification} \\
    \hline \\
    3             & $0.6801 \pm 0.0074$ & $0.7791 \pm 0.0093$ & $0.6251 \pm 0.0016$ \\
    4             & $0.9143 \pm 0.0043$ & $0.9785 \pm 0.0046$ & $0.8291 \pm 0.0899$ \\
    5             & $0.9823 \pm 0.0006$ & $0.9207 \pm 0.0049$ & $0.7936 \pm 0.0805$ \\
    20             & $0.9983 \pm 0.00002$ & $0.9732 \pm 0.0003$ & $0.9842 \pm 0.0015$ \\
    50             & $0.9997 \pm 0.000001$ & $0.9736 \pm 0.00012$ & $0.9982 \pm 0.000039$ \\
    100             & $0.9999 \pm 0.0000003$ & $0.9735 \pm 0.00013$ & $0.9995 \pm 0.0000085$ \\
    \end{tabular}
    \end{center}
  \end{table}
\label{sec:experiment}

\end{document}